\def\year{2017}\relax
\documentclass[letterpaper]{article}
\usepackage{aaai17}
\usepackage{times}
\usepackage{helvet}
\usepackage{courier}
\usepackage{amssymb,amsmath,amsfonts}
\usepackage{mathrsfs}
\usepackage{color}
\usepackage{tikz}
\usepackage{amsgen,amsfonts,amsbsy,amsthm}
\usepackage{amssymb}
\usepackage{amsmath}
\usepackage{esint}

\usepackage{graphicx} 
\usepackage{subfigure} 

\usepackage{algorithm}
\usepackage{algorithmic}

\newtheorem{theorem}{Theorem}

\newtheorem{lemma}{Lemma}

\newtheorem{definition}{Definition}

\newcommand{\norm}[1]{\left\lVert#1\right\rVert}

\newcommand{\bydef}{:=}

\begin{document}
%
\title{Misspecified Linear Bandits}
\author{Avishek Ghosh\\
University of California, Berkeley\\
California 94720 USA\\
avishek\_ghosh@berkeley.edu
\And
Sayak Ray Chowdhury\\
Indian Institute of Science\\
Bengaluru 560012 India\\
srchowdhury@ece.iisc.ernet.in
\And
Aditya Gopalan\\
Indian Institute of Science\\
Bengaluru 560012 India\\
aditya@ece.iisc.ernet.in
}

\maketitle
\begin{abstract}
We consider the problem of online learning in misspecified
linear stochastic multi-armed bandit problems. Regret guarantees for
state-of-the-art linear bandit algorithms such as Optimism in the Face
of Uncertainty Linear bandit (OFUL) hold under the assumption that the
arms’ expected rewards are perfectly linear in their features. It is,
however, of interest to investigate the impact of potential
misspecification in linear bandit models, where the expected rewards
are perturbed away from the linear subspace determined by the arms’
features. Although OFUL has recently been shown to be robust to
relatively small deviations from linearity, we show that any linear
bandit algorithm that enjoys optimal regret performance in the
perfectly linear setting (e.g., OFUL) must suffer linear regret under a
sparse additive perturbation of the linear model. In an attempt to
overcome this negative result, we define a natural class of bandit
models characterized by a non-sparse deviation from linearity. We
argue that the OFUL algorithm can fail to achieve sublinear regret
even under models that have non-sparse deviation. We finally develop
a novel bandit algorithm, comprising a hypothesis test for linearity
followed by a decision to use either the OFUL or Upper Confidence
Bound (UCB) algorithm. For perfectly linear bandit models, the
algorithm provably exhibits OFUL’s favorable regret performance, while
for misspecified models satisfying the non-sparse deviation property,
the algorithm avoids the linear regret phenomenon and falls back on
UCB’s sublinear regret scaling. Numerical experiments on synthetic
data, and on recommendation data from the public Yahoo! Learning to
Rank Challenge dataset, empirically support our findings.
\end{abstract}

\section{Introduction}
\label{sec:intro}

Stochastic multi-armed bandits have been used with significant success
to model sequential decision making and optimization problems under uncertainty, due to
their succinct expression of the exploration-exploitation
tradeoff. Regret is one of the most widely studied performance
measures for bandit problems, and it is well-known that the optimal
regret that can be achieved in an iid stochastic bandit instance
with $N$ actions, $[0,1]$-bounded rewards and $T$ rounds, without any
additional information about the reward distribution, is\footnote{The
  notation $\tilde{O}$ hides polylogarithmic factors.}
$\tilde{O}(\sqrt{NT})$. This is achieved, for instance, by the
celebrated Upper Confidence Bound (UCB) algorithm of
\cite{AuerEtAl02FiniteTime}.

The (polynomial) dependence of the regret in a standard stochastic
bandit on the number of actions $N$ can be rather prohibitive in settings
 with a very large number (and potentially infinite) of actions. Under the assumption that the
rewards from playing arms are linear functions of known features or context vectors,
 linear bandit algorithms such as LinUCB \cite{Li2010}, {\bf
  O}ptimism in the {\bf F}ace of {\bf U}ncertainty {\bf L}inear bandit
(OFUL) \cite{AbbYadPS11} and Thompson sampling \cite{Tho33} give
regret $\tilde{O}(d\sqrt{T})$ where $d$ is the feature dimension. This
is particularly attractive in practice where the feature dimension $d
\ll N$ (for instance, news article recommendation data typically has
$d$ of the order of hundreds while $N$ is $2$ or $3$ orders
higher). The framework also extends to the more general contextual
linear bandit model, where the features for arms are allowed to vary
with time \cite{ChuLiReySch2011:lincontext,agrawal2013thompson}.

The design, and attractiveness, of linear bandit algorithms hinges on the assumption
that the expected reward from playing arms are linear in their
features, i.e., under a fixed ordering of the arms, the vector of
expected rewards from all arms belongs to a known linear
subspace, spanned by the arms' features. However, real-world
environments may not necessarily conform perfectly to this linear
reward model and in fact in most cases, have large deviation
 (Section~\ref{sec:simulation} presents a case study using
a real-world dataset to this effect). One possible reason for this is
that features are often designed with careful domain expertise without
explicit regard for linearity with respect to the utilities of
actions. Another situation where linearity ma be violated is when
there is feature noise or uncertainty \cite{hainmueller2014kernel} --
even a small amount of noise in the assumed features shifts the
expected reward vector out of the linear subspace. When the rewards
need not be perfectly linear in terms of the features in
hand, it becomes important to study how robust or fragile strategies
for linear bandits can be to such misspecification.

The specific questions we address are: (a) With features available for
arms with respect to which the arms' rewards need not necessarily be
linear, how do deviations from linearity impact the performance of state-of-the-art linear bandit algorithms?  (b) Is it possible to design bandit algorithms that control for deviations from linearity and still enjoy `best-of-both-worlds' regret performance, i.e., regret that is sublinear in $T$ and depends only on the feature dimension when the model is linear (or near-linear), and that falls back on the number of arms (as for UCB) when all bets are off (i.e., the model is far from linear)?

\begin{table}[t!]
\centering
\small
\begin{tabular}{|c |c |c |c|}
\hline
Deviation from & OFUL & UCB & RLB \\ 
linearity & & & (proposed) \\
\hline
Small  & $O(d\sqrt{T})$ & $O(\sqrt{NT})$ & $O(d\sqrt{T})$ \\ \hline
Large \& non-sparse   &  $\Omega(T)$ & $O(\sqrt{NT})$ & $O(\sqrt{NT})$ \\ \hline
\end{tabular}
\caption{Regret of OFUL, UCB and the proposed algorithm (RLB) upto time horizon $T$ under different deviations. We can see that RLB avoids linear regret of OFUL for large non-sparse deviations while enjoying the favorable regret of OFUL under very small deviations.}
\label{table:regret_algo}
\end{table}

{\bf Overview of results.} The paper makes the following contributions: 
\begin{enumerate}
\item  We first prove a negative result about the robustness of linear bandit algorithms to sparse deviations from linearity (Theorem~\ref{thm:general_lower_bound}): \emph{Any} linear bandit algorithm that enjoys optimal regret guarantees on perfectly linear bandit problem instances (i.e., $O(d\sqrt{T})$ regret in dimension $d$), such as OFUL and LinUCB, must suffer linear regret on some misspecified linear bandit model. Furthermore our constructive argument shows that it is possible to find a misspecified model that differs only sparsely from a perfectly linear model -- in fact, by a perturbation of the expected reward of only a single arm. We also rule out the possibility of using a state-of-the art bandit algorithm OFUL for handling instances with large
non-sparse deviation (Theorem~\ref{lem:OFUL_large_dev}).

  \item Towards overcoming this negative result, we propose and
    analyze a novel bandit algorithm (Algorithm
    \ref{algo:proposed_scheme}) (abbreviated RLB in Table~\ref{table:regret_algo}),
     which is not only robust to non-sparse deviations
    from linearity but also retains the order-wise optimal regret
    performance in the standard linear bandit model. The algorithm
    provably achieves OFUL's $\tilde{O}(d \sqrt{T})$
    regret\footnote{Note that we concern ourselves with studying the gap-independent (worse-case over problem instances) regret; a similar exercise can be carried out in terms of the reward gap
      parameter. } in the ideal linear case, and UCB's
    $\tilde{O}(\sqrt{NT})$ regret for a broad class of reward models
    which are not linear but are well-separated from the feature
    subspace in a non-sparse sense, which we characterize (Theorem
    \ref{thm:total_regret}). The algorithm is comprised of a
    hypothesis test, followed by a decision to employ either OFUL or
    UCB. Numerical experiments on both synthetic as well as on the
    public Yahoo!  Learning to Rank Challenge data \footnote{{\tt
        {\tiny
          https://webscope.sandbox.yahoo.com/catalog.php?datatype=c}}},
    lend support to our theoretical results.
\end{enumerate}

{\bf Related work.} 
Many strategies have been devised and studied for stochastic
multi-armed bandits for the general setting without structure -- UCB
\cite{AuerEtAl02FiniteTime}, $\epsilon$-greedy
\cite{Cesa-Bianchi98finite-timeregret_epsilon_greedy}, Boltzmann
exploration \cite{sutton1998reinforcement}, Bayes-UCB
\cite{Kaufmann12onbayesian}, MOSS \cite{audibert2009minimax_moss}
and Thompson sampling \cite{Tho33,AgrawalG,Kaufmann2012}, to name a
few. Linear stochastic bandits have been extensively investigated
\cite{rusmevichientong2010linearly,DaniHayKak08,AbbYadPS11} under the
well-specified or perfectly linear reward model, achieving (near)
optimal problem-independent regret of $\tilde{O}(d\sqrt{T})$ if the
features are of dimension $d$ (note that the number of arms can in
principle unbounded). Researchers have also considered extensions of
linear-bandit algorithms for the case of rewards following a
generalized linear model with a known, nonlinear link function
\cite{Filippi08}.

In contrast to the abundance of work on linear bandits, very
little work, to the best of our knowledge, has dealt with the impact
of misspecification on stochastic decision making with partial
(bandit) feedback. A notable study is that of
\cite{besbes2015surprising} who study misspecified models in a
specific dynamic pricing setting. Working in a specialized
$2$-parameter linear reward setting, they arrive at the conclusion
that, within a small range of perturbations of the model away from
linearity, one can preserve the sublinear regret of a standard bandit
algorithm. There has been significant work, in a different vein, on
the effect of model misspecification for the classical linear
regression problem (i.e., estimation) in statistics where the metric
is overall distortion and not explicitly maximum reward -- see for
instance the work of \cite{white1981consequences} and related
references. Very recently \cite{gopalan2016low} provides some results for the linear bandit algorithm OFUL when the devation fron linearity is small. We expect to contribute towards filling a much-needed gap in the study of sensitivity properties in linearly parameterized
bandit decision-making in this work.

\section{Setup \& Preliminaries} \label{sec:system_model}

Consider a multi-armed bandit problem with $N$ arms, and a $d$-dimensional ($d \ll N$)
context or feature vector $x_i \in \mathbb{R}^d$ associated with each
arm $i$, $i = 1, \ldots, N$. An arm $i$, upon playing, yields a
stochastic and independent reward with expectation $\mu_i$. Let
$\mu^*=\max_i \mu_i$ be the best expected reward, and let $\mathcal{X}$ be the matrix having the feature vectors for each arm as its columns:
$\mathcal{X}=[x_1 \; | \; x_2 \; |\; \ldots \; | \; x_N] \in
\mathbb{R}^{d \times N}$, with $\mathcal{X}^T$ assumed to have full column rank. Define $\mu = [\mu_1 \; \mu_2 \; \ldots \;
\mu_N]^T \in \mathbb{R}^{N}$ to be the expected reward vector. 

At each time instant $t = 1, 2, \ldots$, the learner chooses
any one of the $N$ arms and observes the reward collected from that
arm. The action set for the player is
$\mathcal{A}=\lbrace1,2,\ldots,N \rbrace$.  The regret after $T$ rounds is defined to be the quantity  $R(T)=T\mu^*-\sum_{t=1}^{T} \mu_{A_t}$. The goal of the player is
to maximize the net reward, or equivalently, minimize the regret, over the course of $T$ rounds. (If the learner has exact knowledge of $\theta^*$ and $\epsilon$ beforehand, the optimal choice is to play
 a best possible arm $i^*=\arg\max_i \mu_i$ at all
time instances.) 

Under a perfectly linear model, the observed reward $Y_t$ at
time $t$ is modeled as the random variable, $Y_t= \, \left\langle
  x_{A_t}, \theta^* \right\rangle +\eta_t = \mu_{A_t} + \eta_t$, where
$A_t$ is the action chosen at time $t$, $\theta^* \in \mathbb{R}^d$ is
 the unknown parameter vector, $\left\langle.,.\right\rangle$
denotes the inner product in $\mathbb{R}^d$ and $\eta_t$ is zero-mean stochastic noise assumed to be
{conditionally $R$-sub-Gaussian} given $A_t$. Thus, under a perfectly linear model, the mean reward for each arm is a linear function of its features: there exists a unique $\theta^* \in \mathbb{R}^d$ such that
$\mu_i = x_i^T \theta^*$ $\forall i \in \mathcal{A}$ (the uniqueness property follows from the full column rank of $\mathcal{X}^T$).

Consider now the case where a linear model for $\mu$ with respect to the features $\mathcal{X}$ may not be
valid, resulting in a deviation from linearity or a {\em misspecified} linear bandit model. We model the reward in this case by
\begin{equation}
  Y_t= \, \left \langle x_{A_t}, \theta \right\rangle + \epsilon_{A_t} + \eta_t  = \mu_{A_t} + \epsilon_{A_t} + \eta_t, \nonumber
\end{equation}
where $\theta \in \mathbb{R}^d$ is a choice of weights, and $\epsilon \bydef [\epsilon_1 \; \epsilon_2 \; \ldots \;
\epsilon_N]^T \in \mathbb{R}^N$ denotes the deviation in the expected
rewards of arms. Note that (a) the model remains perfectly linear
if\footnote{For a matrix $M$, span($M$) denotes the subspace spanned
  by the columns of $M$.}  $\epsilon \in \mbox{span}(\mathcal{X}^T)
\subseteq \mathbb{R}^N$), and (b) choice of $\theta$ satisfying the equation above is not unique if $\mu$ is separated from the subspace $\mbox{span}(\mathcal{X}^T)$, i.e., $\min_{\theta \in \mathbb{R}^d} \parallel\mathcal{X}^T \theta -\mu \parallel_2 > 0$. 


\section{Lower Bound for Linear Bandit Algorithms under Large Sparse Deviation} \label{sec:general_lower_bound}
In this section, we present our first key result -- a general lower bound on regret of any `optimal' linear bandit algorithm on misspecified problem instances. Specifically, we show that any linear bandit algorithm that enjoys the optimal $O(d\sqrt{T})$ regret scaling, for linearly parameterized models of dimension $d$, must in fact suffer linear regret under a misspecified model in which only one arm has a mismatched expected reward. 

\begin{theorem}
Let $\mathbb{A}$ be an algorithm for the linear bandit problem, whose expected regret is $\tilde{O}(d\sqrt{T})$ on any linear problem instance with feature dimension $d$, time horizon $T$ and expected rewards bounded in absolute value by $1$. There exists an instance of a sparsely perturbed linear bandit, with the expected reward of one arm having been perturbed, for which \textbf{$\mathbb{A}$} suffers linear, i.e., $\Omega(T)$, expected regret.
%
\label{thm:general_lower_bound}
\end{theorem}
 The formal proof of Theorem \ref{thm:general_lower_bound} is deferred to the appendix, but we present the main ideas in the following. \\
 
\noindent {\bf Proof sketch.} The argument starts by considering a perfectly linear bandit instance with order of $\sqrt{T}$ arms in dimension $d$. It follows from the regret hypothesis that number of suboptimal arm plays must be $O(\sqrt{T})$. By a pigeonhole argument, since there are order of $\sqrt{T}$ suboptimal arms, there must exist a suboptimal arm that is played no more than $O(1)$ times in expectation. Markov's inequality then gives that the event that both a) this suboptimal arm is played at most $O(1)$ times {\em and} b) overall regret is $O(d\sqrt{T})$, occurs with probability at least a constant, say $1/3$. 

Having isolated a suboptimal arm that is played very rarely by the algorithm (note that the choice of such an arm may very well depend on the algorithm), the argument proceeds by adding a perturbation to this suboptimal arm's reward to make it the best arm in the problem instance. A change-of-measure argument is now used to reason that in the perturbed instance, the probability of the algorithm playing the arm in question does not change significantly as it was anyway played only a constant number of times in the pure linear model. But this must imply that the expected regret is linear due to neglecting the optimal arm in the perturbed problem instance.

\section{Performance of OFUL Under Deviation} \label{sec:OFUL_algorithms}  A
state-of-the-art algorithm for the linear bandit problem is OFUL. We study the performance of OFUL\footnote{We
  consider the OFUL algorithm in this work chiefly because it is known
  to be the most competitive in terms of regret scaling. It is conceivable that similar results can be
  shown for other, related, bandit strategies as well, such as
  ConfidenceBall \cite{DaniHayKak08}, UncertaintyEllipsoid
  \cite{rusmevichientong2010linearly}, etc.}  for various cases of
deviations $\epsilon$ (suitably ``small'' and
``large''). Specifically, we argue that OFUL is robust to small
deviations, but for large deviations, the performance of OFUL is very
poor leading to a linear regret scaling. The findings motivate us to
propose a more robust algorithm to tackle linear bandit problems with
significantly large deviations.

At time $t \geq 1$, based on previous actions
and observations upto $t-1$, OFUL solves a regularized linear least
squares problem to estimate the unknown parameter
$\theta^* \in \mathbb{R}^d$ and constructs a high-confidence
ellipsoid around the estimate using concentration-of-measure
properties of the sampled rewards. Using the confidence set, the high
probability regret of OFUL is $O(d\sqrt{T})$.

\subsection{OFUL with Small Deviation} \label{subsec:small_deviation}
When the deviation from linearity is considerably small, it can be shown that OFUL performs similar to the perfect linear model in terms of regret scaling (see \cite[Theorem 3]{gopalan2016low} for details and a formal quantification of ``small'' deviation). Assuming $||\theta^*||_2\leq S,  $ $||x_i||_2 \leq L $ and $|\mu_i| \leq 1$ for all $i \in \mathcal{A}$, with probability at least $1-\tilde{\delta}$ ($\tilde{\delta} > 0$), the cumulative regret upto time $T$ of OFUL is given by,
 \footnotesize
  \begin{eqnarray}
    R_{OFUL}(T) \leq 8 \rho' \sqrt{Td \log\left(1 +
      \frac{TL^2}{\lambda d}\right)} \Big (\lambda^{1/2}S  \nonumber \\
    +R\sqrt{2 \log\frac{1}{\tilde{\delta}} + d \log
      \left(1+\frac{TL^2}{\lambda d}\right)} \Big) \nonumber
        \end{eqnarray}
  \normalsize
where $\rho'$ is a geometric constant that measures the ``distortion'' in the arms' actual rewards with respect to (linear) approximation and $\lambda$ is a regularization parameter.

\textbf{\textit{Remark:}} OFUL retains $O(d\sqrt{T})$ regret scaling even in the presence of ``small'' deviation.
\subsection{OFUL with Large Sparse Deviation} \label{subsec:OFUL_large_sparse}
The regret of OFUL under pure linear bandit instance is $O(d\sqrt{T})$. Therefore from Theorem~\ref{thm:general_lower_bound}, the cumulative expected regret under large sparse deviation will be $\Omega(T)$.
\subsection{OFUL with Large Non-sparse Deviation} \label{subsec:linear_regret}
 We need to identify a natural class of structured large deviations that we dub {\em non-sparse}. We impose the following structure in terms of sparsity on the expected
rewards $\mu$. Recall from
Section~\ref{sec:system_model} that $\mathcal{X}$ denotes the context
matrix, $\mu$ the mean reward vector, $\theta$  a choice of weights, and
$\epsilon$ the deviation from mean $\mu$; thus, $\mu=\mathcal{X}^T
\theta + \epsilon$. 

\begin{definition} [{\bf Non-sparse deviation}] 
Given a feature set $X^{f} = \lbrace x_1, ..., x_N \rbrace \subset R^d $ and constants $ l > 0$, $ \beta \in [0,1]$, an expected reward vector $ \mu \in R^N$ is said to have the $(l,\beta)$ deviation property if,
 \begin{equation}
\mathbb{P} \big (|x_{i_{d+1}}^T[X^{f}_{i_1,\ldots,i_d}]^{-1}[\mu_{i_1,\ldots,i_d}]-\mu_{i_{d+1}}| \geq l \big ) \geq 1-\beta \nonumber
\end{equation}

\noindent for all $\lbrace i_1,i_2,\ldots,i_d,i_{d+1} \rbrace
\subseteq \lbrace 1,2,\ldots,N \rbrace $, such that $\lbrace x_{i_1}, x_{i_2}, \ldots, x_{i_d}
\rbrace$ linearly independent, where
$ X^{f}_{i_1,\ldots,i_d}=[x_{i_1}^T,\ldots, x_{i_{d}}^T]^T$ and
$\mu_{i_1,\ldots,i_d}=[\mu_{i_1} , \ldots, \mu_{i_{d}}]^T$. The randomness is over the choice of $d+1$ arms.
\label{def:sparse}
\end{definition}
In other words, the deviation of reward $\mu$ is $(l,\beta)$ non-sparse if,
whenever one uses any $d$ linearly independent features, with their
corresponding rewards, to regress a $(d+1)$-th unknown reward
linearly, then the magnitude of error is at least $l >
0$ (bounded away from 0) with probability at least $1-\beta$. Typically, $\beta$ is positive and close to $0$. 

For example, consider the problem instance of Theorem~\ref{thm:general_lower_bound}, i.e., only one arm is perturbed away from linearity. This is an example of sparse deviation. If the perturbed arm is picked as one of $d+1$ arms in Definition~\ref{def:sparse}, $l$ will be a large positive number, but when the perturbed arm is missed, $l$ will be 0, which is inconsistent with Definition~\ref{def:sparse}. Also, $\beta$ can be chosen such that the probability of missing the perturbed arm is strictly greater than $\beta$.

We now argue, by counterexample (Theorem~\ref{lem:OFUL_large_dev}), that the regret of OFUL with large
 non-sparse deviation is $\Omega(T)$. 

\begin{theorem}
 Consider a linear bandit problem with $\mathcal{A}=\lbrace 1,2 \rbrace$, context matrix $\mathcal{X}=[1\; 2]$, mean reward vector $\mu=[\mu_1 \,\,\, \mu_2]^T$ with $\mu_2 > \mu_1 $ and $\mu_2 \neq 2\mu_1$. The deviation vector $\epsilon = [\epsilon_1 \,\,\, \epsilon_2]^T$ is such that $|\epsilon_i| > c$ ($c >0$) for $i=\lbrace 1,2 \rbrace$ (with respect to Definition~\ref{def:sparse}, $l=c$ and $\beta=0$). There exists a problem instance for which the expected regret of OFUL until time $T$, $\mathbb {E} (R_{OFUL})=\Omega(T)$.
\label{lem:OFUL_large_dev}
\end{theorem}

 The description of problem instance with the formal proof of theorem is deferred to the supplementary material.

\textbf{\textit{Summary:}} OFUL is robust to ``small'' deviation (irrespective of sparsity) but incurs linear regret under large deviation (for both sparse and non-sparse). Theorem~\ref{thm:general_lower_bound} shows the futility of designing any linear bandit algorithm under sparse deviation. However the quest is still valid if the deviation is large but non-sparse. We will investigate this issue in rest of the paper. It is clear that under large deviation, context vectors do not contribute in reducing regret and thus a rational player should discard contexts under such circumstances. The player may choose any standard algorithm for basic multi-armed bandits (UCB for instance).

\section{A Linear Bandit Algorithm Robust to Large, Non-sparse Deviations} \label{sec:sparsity} 

This section accomplishes the objective of 
developing a new algorithm that maintains the sublinear regret
property in a model with non-sparse, large deviations.
Non-sparse deviations can be seen to
  naturally arise in the presence of stochastic measurement or
  estimation noise; e.g., let $x_i$ and $\bar{x_i}$ be the measured and original context
  vector respectively for arm $i$ with $x_i =\bar{x_i} + \zeta_i$. $\zeta_i^T \theta $ 
  can be modeled as a Gaussian random variable with
  mean, $\mathbb{E}(\zeta_i^T \theta)=\epsilon_i$. Substituting, we get, 
  $\mu=\mathcal{X}^T \theta + \epsilon$. It is possible to find suitable $(l,\beta)$ pair
  (Definition~\ref{def:sparse}) for this model and thus $\mu$ is
  non-sparse. The associated feature vectors corresponding to the mean reward
  vector satisfying Definition~\ref{def:sparse}, are called
  ``uniformly perturbed features''. 
  
 We now define $2$ hypotheses -- $\mathbf{\mathcal{H}_0}$ and $\mathbf{\mathcal{H}_1}$, corresponding intuitively to ``linear'' and ``not linear" -- on $(\mathcal{X}, \mu)$, which will be used to quantify the performance of the algorithm developed in this section. We say that hypothesis $\mathbf{\mathcal{H}_0}$ holds if the separation of $\mu$ from $\mbox{span}(\mathcal{X}^T)$, i.e., the quantity $\min_{\theta \in \mathbb{R}^d} \parallel\mathcal{X}^T \theta -\mu \parallel_2$, is $0$, i.e., the model is perfectly linear. On the other hand, we say that hypothesis $\mathbf{\mathcal{H}_1}$ holds if the separation is greater than $0$ and $\mu$ satisfies the $(l_1,\beta)$ deviation property of Definition~\ref{def:sparse} with $l_1 > 0$.
 
 \textit{\textbf{Remark:}} The definition of $\mathbf{\mathcal{H}_0}$ be generalized to handle small deviations in the $\parallel \cdot \parallel_2$ norm with distortion parameter $\rho' \geq 1$, in the sense of \cite[Theorem 3]{gopalan2016low}.
   
\subsection{A Robust Linear Bandit (RLB) Algorithm} \label{sec:proposed_algo}
 The sequence of actions for the proposed novel bandit algorithm,
namely Robust Linear Bandit (RLB) is summarized in
Algorithm~\ref{algo:proposed_scheme}, mainly consisting of three
steps. First, RLB executes an initial sampling phase, in
which $d+1$ arms out of $N$ are sampled. Based on these samples,
it constructs a confidence ellipsoid for $\theta^*$ in the
next phase. Finally, based on experimentation on the $(d+1)$-th arm,
it decides to play either OFUL or UCB for the remainder of the horizon. We will illustrate the necessity of non-sparse deviation as follows: consider a problem instance with  
$\epsilon=(0,\ldots,0,c,0,\ldots,0)$, $|c|\gg 0$. As $N \gg d$, 
with high probability, the deviated arm can be missed in the sampling 
phase and according to Algorithm~\ref{algo:proposed_scheme}, the learner learns that the model is linear and decides to play OFUL which according to Theorem~\ref{thm:general_lower_bound} incurs $\Omega(T)$ regret.
\subsection*{Step 1: Sampling of $d+1$ arms} \label{subsec:sampling_interval}
 For non-sparse deviation, the choice of $d+1$ among $N$ arms may be arbitrary. Without loss of generality, we sample the arms indexed $\lbrace 1,2,\ldots,d+1 \rbrace $,  $k$ times each (resulting $(d+1) \times k$ ($\bydef \tau$) sampling instances). From Hoeffding's inequality, the sample mean estimate of $d+1$-th arm, $\hat{\mu}_{d+1}$, satisfies $\mathbb{P}(|\hat{\mu}_{d+1}-\mu_{d+1}| > r_s) \leq \exp (-2r_s^2 k )$. With $\delta_s:=\exp (-2 r_s^2 k)$, the confidence interval around $\mu_{d+1}$ will be $[\mu_{d+1}-r_s, \mu_{d+1} + r_s] $ with probability at least $1-\delta_s$.
\subsection*{Step 2: Construction of Confidence Ellipsoid} \label{subsec:conf_ellipsoid}
  Based on the samples of first $d$ arms, RLB constructs a confidence ellipsoid for $\theta^*$ assuming $\mathcal{H}_0$ is true. Under $\mathcal{H}_0$,
  \begin{equation}
  y_i^{(j)}= \left\langle x_i,\theta^* \right\rangle + \eta_{i,j} \, \, \forall i\in \lbrace 1,2,\cdot,d \rbrace, 1 \leq j \leq k. \nonumber
  \end{equation}
  
In this setup, we re-define the reward vector $\mathbf{Y} =[ y_1^{(1)},\cdot, y_1^{(k)}, y_2^{(1)},\cdot,y_2^{(k)}, \ldots ,y_d^{(k)}]^T$, feature-matrix $\mathbf{X}= [x_1^T, x_1^T, \cdot, x_1^T , x_2^T, x_2^T, \ldots, x_d^T]^T$ and noise vector $\eta=[\eta_1, \eta_2,  \ldots ,\eta_{k d}]^T$ with $\mathbf{Y}=\mathbf{X} \theta^* + \eta$. Let $\hat{\theta}$ be the solution of $ \ell^2$ regularized least square, i.e., $\hat{\theta}=(\mathbf{X}^T \mathbf{X} + \lambda I)^{-1}\mathbf{X}^T\mathbf{Y}$, where $\lambda>0 $ is the regularization parameter.

Using the same line of argument as in \cite{AbbYadPS11}, it can be shown that for any $\bar{\delta} >0$, with probability at least $1-\bar{\delta}$, $\theta^*$ lies in the set,
\footnotesize
\begin{eqnarray}
C &=& \bigg \lbrace \theta \in \mathbb{R}^d : \norm {\hat{\theta}- \theta}_{\bar{V}} \leq  R  \nonumber  \\
&&\times \sqrt{2 \log  (\frac{\det(\bar{V})^{1/2} \det (\lambda I)^{-1/2}}{\bar{\delta}})} + \lambda^{1/2}S \bigg \rbrace \nonumber
\end{eqnarray}
\normalsize
$\bar{V}=\lambda I + k \sum_{i=1}^{d} x_i x_i^T= \mathbf{X}^T \mathbf{X}+\lambda I$, $\norm {\theta^*}_2 \leq S $.
\subsection*{Step 3: Hypothesis test for non-sparse deviation}
We project the confidence ellipsoid onto the context of $d+1$-th arm. The projection, $\left\langle x_{d+1}, \theta \right\rangle, \theta \in C$  will result in an interval, $I_e$, centered at $ x_{d+1}^T \hat{\theta}$ (Lemma~\ref{lem: conf_interval_hyp0}). We compare $I_e$ with the interval obtained from sampling $d+1$-th arm, $I_s$. If $\mathcal{H}_0$ is true, Lemma~\ref{lem:low_false_alarm} states that $I_e$ and $I_s$ 
overlap with high probability. Similarly, from Lemma~\ref{lem:low_miss_prob}, under $\mathcal{H}_1$, $I_e$ will not intersect with $I_s$ with high probability, i.e., probability of choosing $\mathcal{H}_1$ when $\mathcal{H}_0$ is true (and vice versa), is significantly low. \footnote{Owing to space constraints, Lemma~\ref{lem: conf_interval_hyp0}, \ref{lem:low_false_alarm} and \ref{lem:low_miss_prob}, with their proofs are moved to supplementary material.} Based on this experiment, the player adopts the following decision rule: 
if $ I_e \cap I_s \neq \phi$, declare $\mathcal{H}_0$ and play OFUL, otherwise declare $\mathcal{H}_1$ and play UCB.
\begin{algorithm}[!hbtp]
	\caption{Robust Linear Bandit (RLB)} \label{algo:proposed_scheme}
	\begin{algorithmic}[1]	
	\STATE Sample the first $d$ arms $k$ times each.
	\STATE Compute the $\ell^2$-regularized least square estimate ($\hat{\theta}$) based on $d \times k$ samples assuming $\mathcal{H}_0$.
	\STATE Construct a confidence ellipsoid $C$ such that with high probability, $\theta^* \in C$.
    \STATE Project the ellipsoid onto the context of $d+1$ th arm to obtain interval $I_e$.	
	\STATE Sample $d+1$ th arm $k$ times, obtain mean estimate, $\hat{\mu}_{d+1}$, and confidence interval $I_s$.
	\STATE If $ I_e \cap I_s \neq \phi$, declare $\mathcal{H}_0$ and play OFUL for the remaining time instants, otherwise play UCB.
	\end{algorithmic} 
\end{algorithm}
\section{Regret Analysis} \label{sec:regret}
The objective of RLB is to learn the gap from linearity and play accordingly to obtain regret of Table~\ref{table:regret_algo}. For zero deviation, RLB exploits linear reward structure and incur a regret of $O(d\sqrt{T})$. For large non-sparse deviation, RLB discards the contexts and avoids linear regret. During the initial sampling phase upto $ \tau $, regret will scale linearly as each step is either forced exploration or exploitation, i.e., $R_s(\tau)= O(\tau)$. After that, based on the player's decision, either OFUL or UCB is played. For $\mathcal{H}_0$, we  use regret of OFUL as given in \cite{AbbYadPS11}. With $N$ arms and time $T$, \cite{bubeck2012regret} provided $O(\sqrt{NT\log T})$ regret for standard UCB. Also, \cite{audibert2009minimax_moss}, gave an algorithm MOSS, inspired by UCB which incurs a regret upper bound of $49 \sqrt{NT}$. 
\subsection{Regret of Algorithm \ref{algo:proposed_scheme}} \label{subsec:regret_proposed}
 From Lemma~(\ref{lem:low_false_alarm}) it can be seen that, if $\mathcal{H}_0$ is true, OFUL and UCB are played with a probability of $1-\delta_1(k,\lambda)$ and $\delta_1(k,\lambda)$ respectively and accordingly regret is accumulated. By an appropriate choice of $k$ and $\lambda$, $\delta_1(k,\lambda)$ can be made arbitrarily close to $0$. Similarly, under $\mathcal{H}_1$, corresponding probabilities are $\delta_2(k,\lambda)+ \beta $ and $1-\delta_2(k,\lambda) -\beta$ respectively (Lemma~\ref{lem:low_miss_prob}). $\beta$ comes from the definition of non-sparse deviation. Therefore, under non-sparse deviation, probability of playing OFUL and incurring linear regret is $\delta_2(k,\lambda)+ \beta $, which can be pushed to arbitrarily small value by proper choice of $k$ and $\lambda$ as typically, $\beta$ is very small and close to $0$. $\tau$ can be choosen as $\log(T)$, a sub-linear function of $T$. Now we are in a position to state our main result - an upper bound on regret of RLB.

\begin{theorem}[Regret guarantees for RLB]
  The expected regret of RLB in $T$ time steps
  satisfies the following: 
(a) Under hypothesis $\mathcal{H}_0$,
  \footnotesize
\begin{eqnarray}
&& \mathbb{E}(R_{RLB}(T)) \leq c_1  ( (d+1)k ) + 4 [ (1-\delta_1(k,\lambda)) \nonumber \\
 &&\times  \sqrt{(T-\log T)d \log (1 +
      \frac{(T-\log T)L^2}{\lambda d})}  (\lambda^{1/2}S  \nonumber \\
    &+& R\sqrt{2 \log\frac{1}{\delta} + d \log
      (1+\frac{(T-\log T)L^2}{\lambda d})} ) ] \nonumber \\
      &+& 49 \delta_1(k,\lambda)  \sqrt{N(T-\log T)} \nonumber
\end{eqnarray}
\normalsize
(b) Under $\mathcal{H}_1$,
\footnotesize
\begin{eqnarray}
 \mathbb{E}(R_{RLB}(T)) \leq c_1  ( (d+1)k )+ 49 (1-\delta_2(k,\lambda)-\beta) \nonumber \\
 \times \sqrt{N(T-\log T)} + c_2 (\delta_2(k,\lambda)+\beta)(T-\log T) \nonumber
\end{eqnarray}
\normalsize
where, a total of $d+1$ arms are sampled $k$ times each, $\lambda$ is regularization parameter, $\delta$, $L$, $c_1$, $c_2$ are constants and,
\footnotesize
\begin{eqnarray}
&&\delta_1(k,\lambda) \bydef \exp ( - \frac{k (r_s \sqrt{\log k} + r_p (\sqrt{\log k} -1))^2}{2R^2}) \nonumber \\
&&\delta_2(k,\lambda) \bydef \exp (- \frac{k (l_1 - r_p \sqrt{\log k} - r_s \sqrt{\log k})^2 }{2R^2}) \nonumber 
\end{eqnarray}
\normalsize
with $2 r_s$ and $2r_p$ being the length of the intervals $I_s$ and $I_e$ respectively and $l_1$ comes from the definition of $\mathcal{H}_1$.
\label{thm:total_regret}
\end{theorem}
{\bf Implication.} We see that if $k$ increases, $\delta_1(k,\lambda)$ and $\delta_2(k,\lambda)$ goes to 0 exponentially. Under $\mathcal{H}_1$ and a given $(l,\beta)$ pair, for RLB to decide in favor of $\mathcal{H}_1$ and hence ensuring sub-linear regret with probability greater than $1-\delta_2(k,\lambda)-\beta$, we need, $\sqrt{\log k}(r_p + r_s) < l_1 $, (shown in the proof of Lemma~\ref{lem:low_miss_prob}). Since $r_s$ and $r_p$ are both $O(1/\sqrt{k})$, $k$ satisfies, $k/\log k > b/l_1^2$ for some constant $b$ ($>0$). Simulations show that a considerably small $\lambda$ also pushes $\delta_1(k,\lambda)$ and $\delta_2(k,\lambda)$ close to $0$. Therefore, with $\mathcal{H}_0$, $R_{RLB}(T)=O(\log T) + O( d \sqrt{T-\log T})$. Similarly, for $\mathcal{H}_1$, $R_{RLB}(T)=O(\log T) + O(\sqrt{N(T-\log T)})$, as shown in Table~\ref{table:regret_algo}. 
\section{Simulation Results}  \label{sec:simulation}

\subsection{Synthetic Data}

 In this setup, we assume, $N=1000$, $d=20$ and $k=50$. $\lambda$ and $ R $ are taken as $0.001$ and $0.1$ respectively. Context vectors and mean rewards are generated at random (in the range $[0,1]$). All high probability events are simulated with an error probability of $0.001$. The simulation is run for 1000 instances and cumulative regret is shown in Figure~\ref{fig:regret_plot}.

\begin{figure}[t!]

\centering
\subfigure[]{\includegraphics[height=1.2in,width=1.6in]{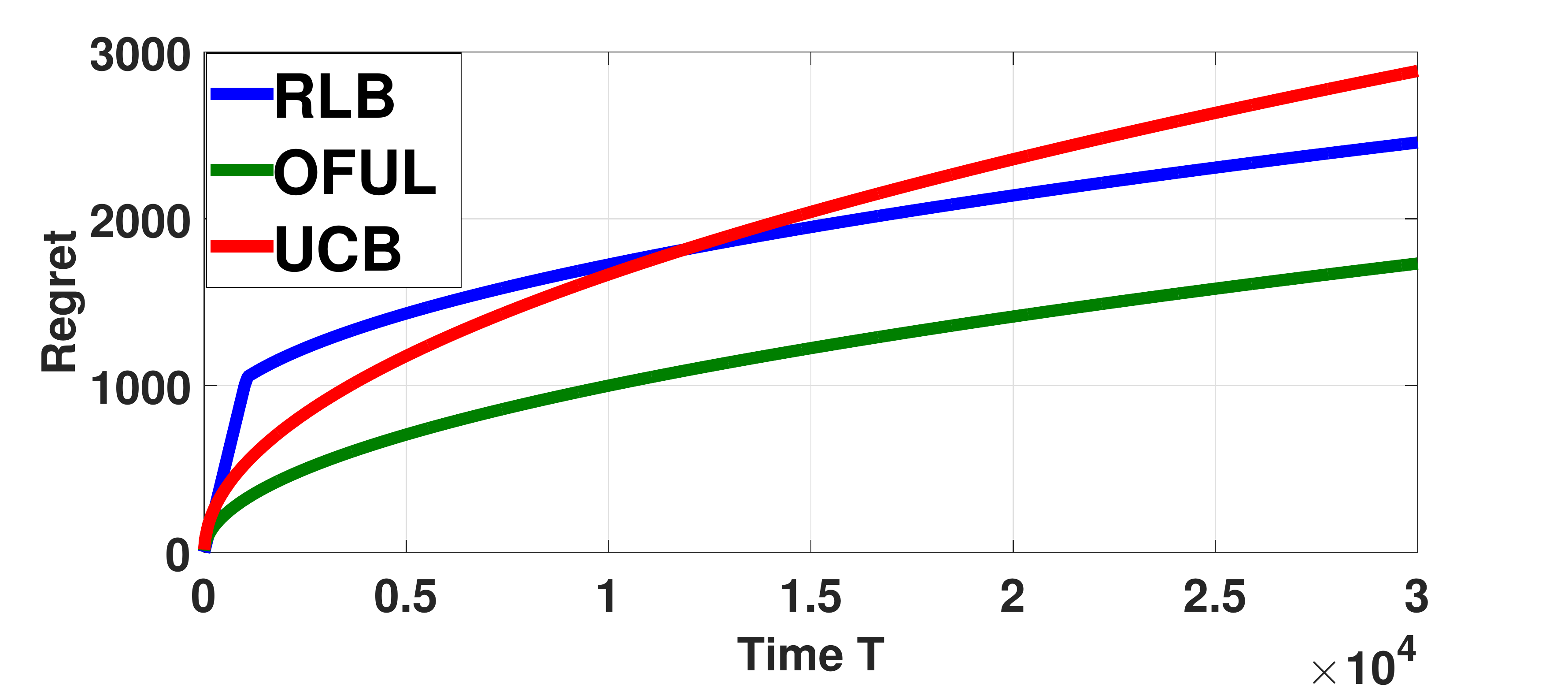}} 
\subfigure[]{\includegraphics[height=1.2in,width=1.6in]{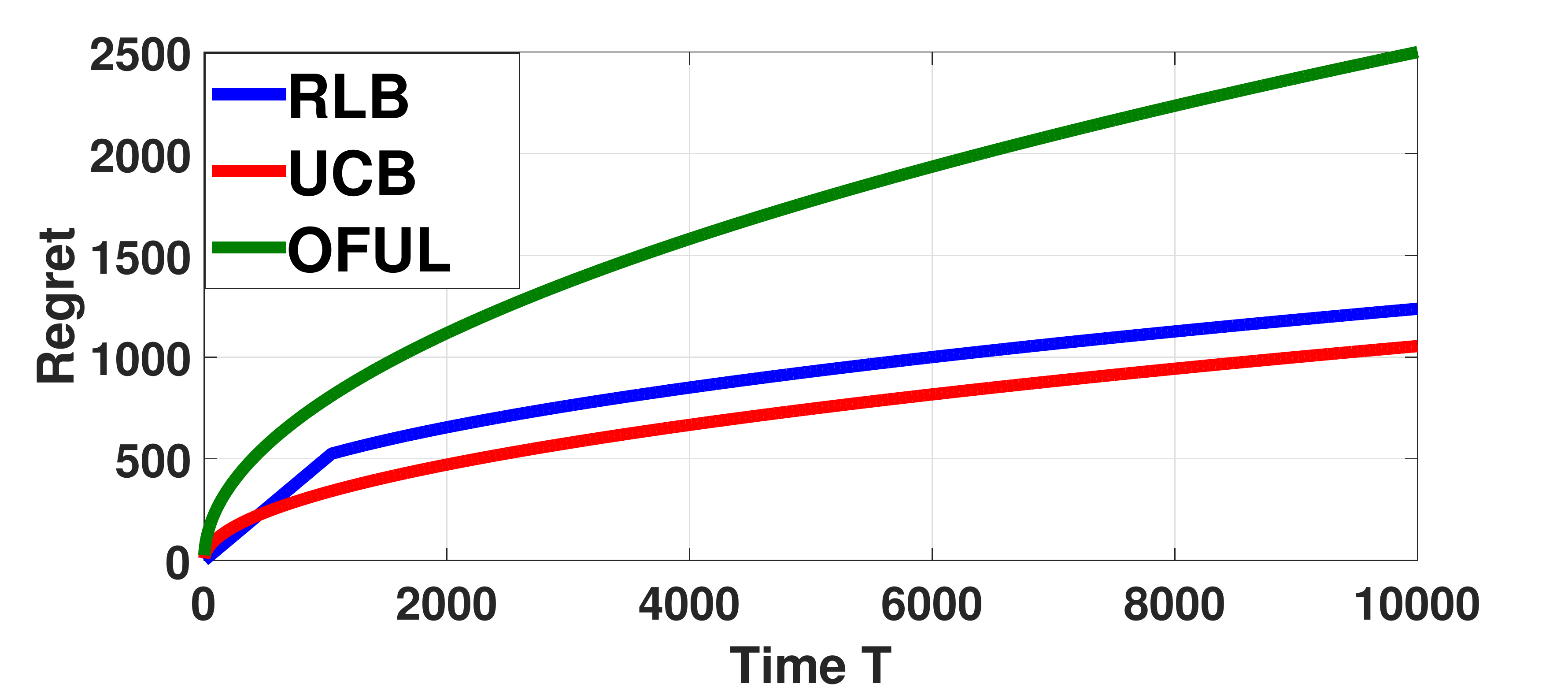}}

 \caption{Regret variations with synthetic data. Figure (a) represents the scenario with 0 deviation, thus regret of RLB follows that of OFUL. In (b), where deviation is non-sparse, RLB avoids the high regret of OFUL and follows UCB.}
\label{fig:regret_plot} 

\end{figure}

 Under $\mathcal{H}_0$, RLB predicts correctly with a probability of false alarm $0.0001$. Figure~\ref{fig:regret_plot} shows the regret performance of RLB. In the sampling phase, regret is linear and thus greater than the perturbed OFUL and UCB algorithm. After the sampling phase, regret of RLB closely follows regret of OFUL with probability $0.9999$. The false alarm probabillity can be further pushed if the value of $k$ is increased. If we allow time horizon $T$ to be very large, the deviation in terms of regret between UCB and RLB will be significantly large.
 
 The same experiment is carried for $\mathcal{H}_1$ with $|\epsilon_i| > 2$ for all $i \in \lbrace 1,2,\ldots,N \rbrace $, and RLB verdicts in favor of UCB with an error (miss detection) of $0.0001$. Figure~\ref{fig:regret_plot} shows the variation of regret with time. Further, if $k$ is increased, the error decreases but the regret from sampling phase increases. 
 
 \subsection{Yahoo! Learning to Rank Data}
  The performance of RLB is evaluated on the Yahoo!
 dataset ``Learning to Rank Challenge''
 \cite{chapelle2011yahoo}. Specifically, we use the file
 {\footnotesize {\tt set2.test.txt}}. The dataset consists of query
   document instance pairs with $103174$ rows and $702$ columns. The first column lists rating given by user (which
   we take as reward) with entries $\lbrace 0,1,2,3,4 \rbrace$ and
   the second column captures user id. We treat the rest $700$ columns as context vector corresponding to each user. We select $20,000$ rows and $50$ columns
   at random (similar results were found for several random
   selections). We cluster the data using $\mathcal{K}$-means
   clustering with $\mathcal{K}=500$. Each cluster can be treated as a
   bandit arm with mean reward equal to the empirical mean of the
   individual rating in the cluster and context (or feature) vector
   equals to the centroid of the cluster. Thus, we have a bandit
   setting with $N=500$, $d=50$.

To show that the obtained data does not fall in $\mathcal{H}_0$ (i.e., linear model), we fit a linear regression model. It is observed that, average value of residuals (error) is $ 0.15 $ (with a maximum value of $0.67$), where average mean reward is $1.13$. Therefore, we conclude that the data falls under $\mathcal{H}_1$. We run OFUL, UCB and RLB on the dataset and regret performance is shown in Figure~\ref{fig:regret_plot_yahoo}. We consider the following cases:

\begin{enumerate}
\item $k=70$: we conclude that all arms are sufficiently sampled and thus RLB avoids high regret of OFUL and plays UCB. But RLB suffers high regret upto $3570$ rounds.
\item $k=10$: arms are not properly sampled, leading to an increase in the radius $I_s$ and violating the lower bound on $k$. Owing to this, RLB plays OFUL and incurs high regret.
\end{enumerate}

\begin{figure}[t!]

\centering
\subfigure[]{\includegraphics[height=1.2in,width=1.6in]{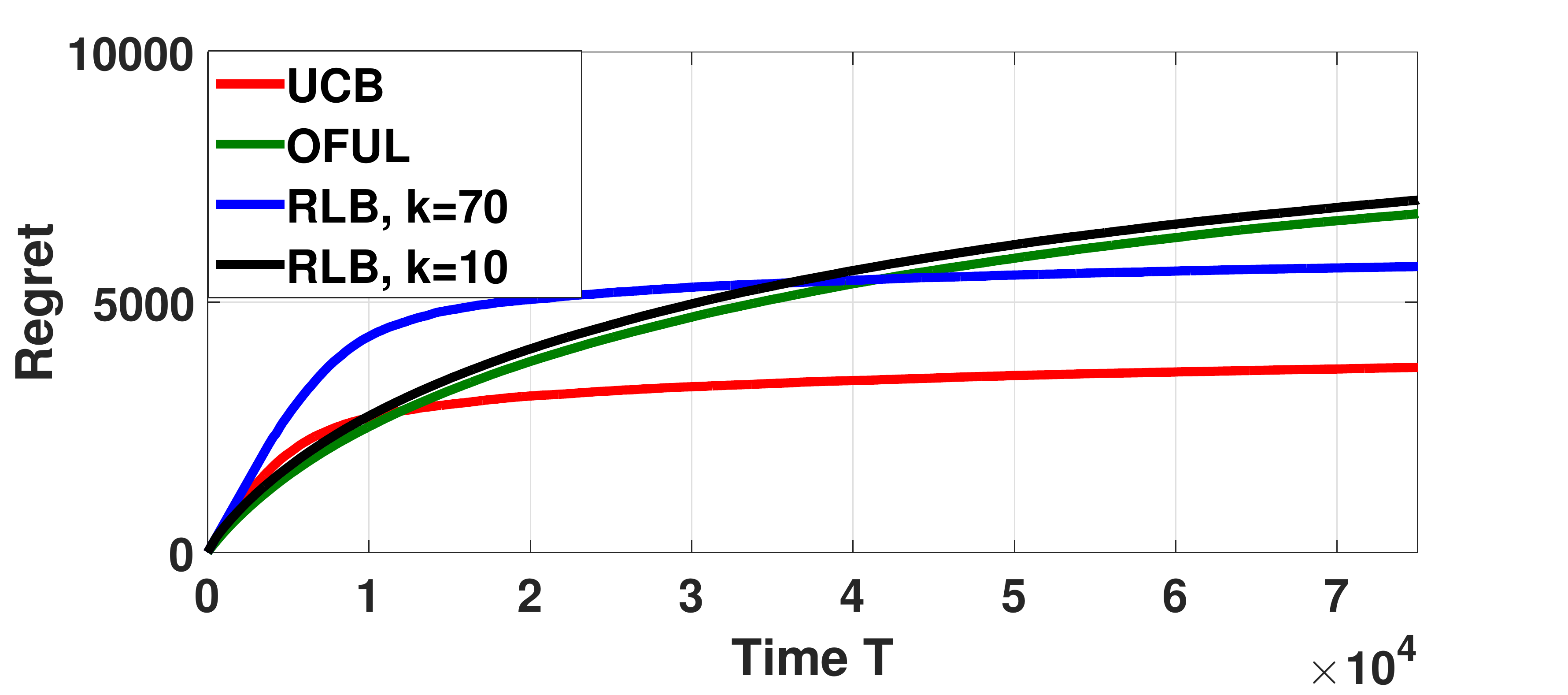}}\subfigure[]{\includegraphics[height=1.2in,width=1.6in]{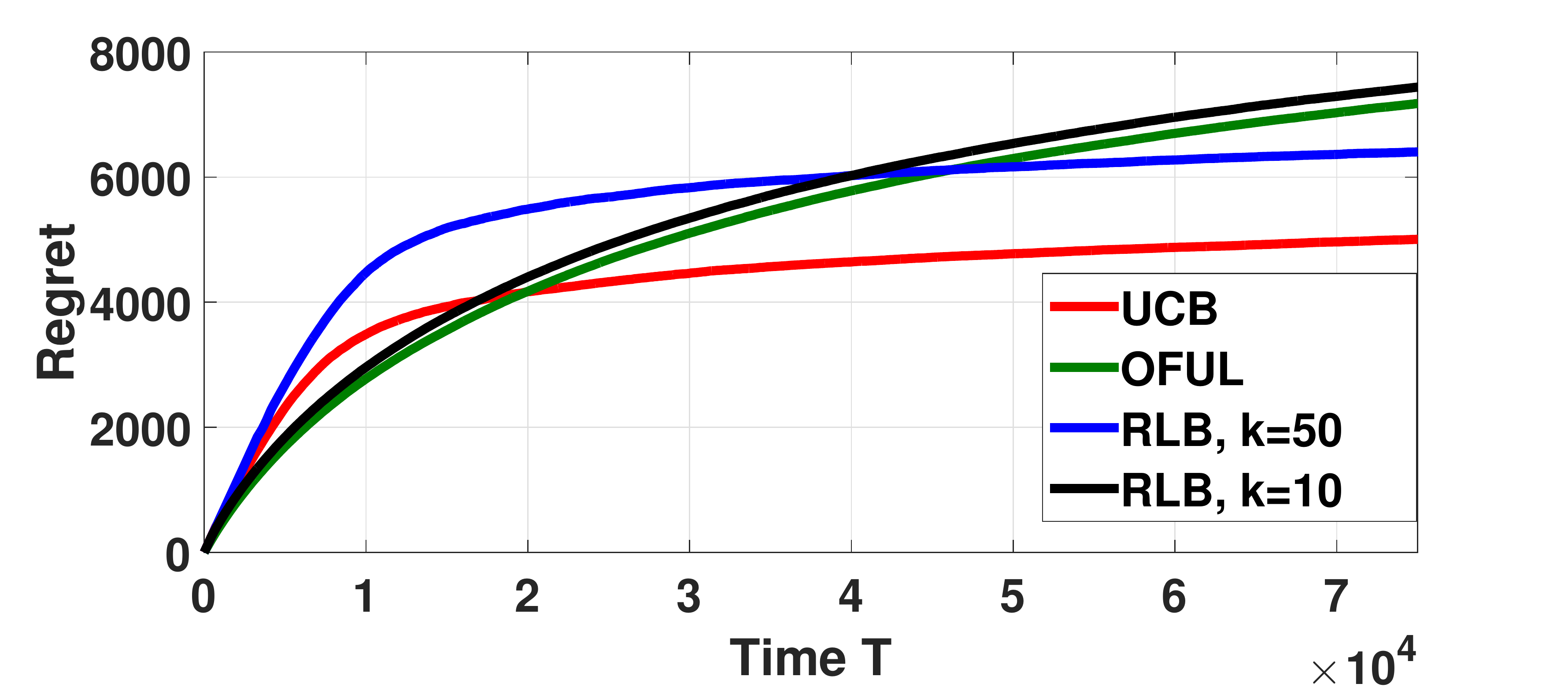}}

 \caption{RLB on Yahoo! data: In Figure (a), $N=500$, $d=50$. Regret of UCB, OFUL and RLB is plotted for different  $k$. Figure (b) denotes similar plots, with $N=800$, $d=50$.}
\label{fig:regret_plot_yahoo} 

\end{figure}

We carry out the same experiment with $\mathcal{K}=800$, i.e., $N=800$, $d=50$ and the observations are similar (Figure~\ref{fig:regret_plot_yahoo}(b)). For a reasonable value of $k$ (50 in this case), RLB properly identifies the optimal algorithm (UCB) to play, but with very low $k$ (10), RLB suffers the high regret of OFUL. We omit the errorbars as over $1000$ instances, regret values for different algorithms remain almost the same.

\section{Conclusion and Future work}
\label{sec:conclusion}

We addressed the problem of adapting to misspecification in linear
bandits. We showed that a state-of-the art linear bandit algorithm
like OFUL is not always robust to deviations away from linearity. To
overcome this, we have proposed a robust bandit algorithm and provided
a formal regret upper bound. Experiments on both synthetic and real
world datasets support our reasoning that (a) feature-reward maps can
often be far from linear in practice, and (b) employing a strategy
that is aware of potential deviation from linearity and tests for it
suitably does lead to performance gains. Moving forward, it would be 
interesting to explore other non-linearity
structures than sparse deviations as was studied here, and to derive
information-theoretic regret lower bounds for the class of general
bandit problems with given feature sets. It is also intriguing to
investigate the performance of Bayesian-inspired algorithms like Thompson
Sampling on linear bandits in presence of deviations.


\newpage

\section*{Acknowledgements}
This work was partially supported by the DST INSPIRE faculty grant IFA13-ENG-69. The authors are grateful to anonymous reviewers for providing useful comments.

\bibliographystyle{aaai}
\bibliography{ghosh}

\newpage


\newpage

\begin{lemma}
(Measure-Change Identity) Consider the probability measures $\mathbb{P}_{lin} [\cdot]$ and $\mathbb{P}_{pert}[\cdot]$ under linear and perturbed bandit model respectively. For any discrete random variable $Y$,
 \begin{equation}
 \mathbb{E}_{pert}[Y] = \mathbb{E}_{lin} \big [ Y \prod_{i=1}^{N_x(T)} \frac{\mathbb{P}_{pert}(r_{x,i})}{\mathbb{P}_{lin}(r_{x,i})} \big ] \nonumber
 \end{equation}

where the arm with feature vector $x$, which is sub-optimal under linear model is boosted to become optimal under perturbed model and $N_x(T)$ denotes the number of plays of the arm upto time $T$. $\mathbb{P}_{[\cdot]}(r_{x,i})$ denotes the probability of obtaining reward $r_x$ in $i$-th round under associated probability measures.
\label{lem:measure_change_identity}
\end{lemma}

\begin{proof}
\footnotesize
\begin{eqnarray}
  \mathbb{E}_{lin} \big [ Y \prod_{i=1}^{N_x(T)} \frac{\mathbb{P}_{pert}(r_{x,i})}{\mathbb{P}_{lin}(r_{x,i})} \big ]=  \sum_{\omega \in \Omega} Y(\omega) \prod_{i=1}^{N_x(T)} \frac{\mathbb{P}_{pert}(r_{x,i})}{\mathbb{P}_{lin}(r_{x,i})} \mathbb{P}_{lin}(\omega) \nonumber 
 \end{eqnarray}
\normalsize
It can be observed that only the reward distribution of the arm with feature $x$ (i.e., $r_x$) is changed under the perturbed probability measure. Hence, for transformation from linear to perturbed probability measure, the reward terms associated with arm $x$ only should be taken care of. As the reward distributions are independent across rounds and the play is continued upto $T$ rounds, over a sample path, we can write:

\begin{eqnarray}
 \mathbb{P}_{pert}(\omega)= \prod_{i=1}^{N_x(T)} \frac{\mathbb{P}_{pert}(r_{x,i})}{\mathbb{P}_{lin}(r_{x,i})} \mathbb{P}_{lin}(\omega) \nonumber 
 \end{eqnarray}

Therefore,

\begin{eqnarray}
  \sum_{\omega \in \Omega} Y(\omega) \prod_{i=1}^{N_x(T)} \frac{\mathbb{P}_{pert}(r_{x,i})}{\mathbb{P}_{lin}(r_{x,i})} \mathbb{P}_{lin}(\omega) &=& \sum_{\omega \in \Omega} Y(\omega) \mathbb{P}_{pert}(\omega)\nonumber \\
  &=& \mathbb{E}_{pert}[Y] \nonumber 
 \end{eqnarray}

\end{proof}

\section*{Proof of Theorem~\ref{thm:general_lower_bound}}

We will construct a perturbed instance with $N$ arms (a function of $T$, to be specified later) and feature dimension $d$.

\paragraph{Construction:}
A pure linear bandit instance is constructed with $N$ bandit arms and dimension $d$ having feature vectors $x_1, x_2, \ldots, x_N$ such that $N=1+c\sqrt{T}$, ($c$ is a constant). The mean rewards on $N$ arms are denoted by $\mu_1, \mu_2, \ldots, \mu_N$. The mean reward of the optimal arm is $\mu^*$ ($=\max_{i} \mu_i$) with feature vector $x_{lin}^*$. We assume all the rewards are bounded in $[0,1]$. Define $\Delta_i= \mu^*-\mu_i$ and $\Delta=\min_i \Delta_i$. We also assume that $\Delta$ is close to $1$, i.e., the mean rewards of sub-optimal arms are close to one another and the mean reward of the optimal arm is much higher than each one of them. Under such construction, the expected regret is given by,

\begin{eqnarray}
\mathbb{E}_{lin}[R_{lin}(T)] &=& \mathbb{E}_{lin} \big [ \sum_{x_i \neq x_{lin}^*} \Delta_i N_{x_i}(T) \big ] \nonumber \\
& \leq & \mathbb{E}_{lin} \big [ \sum_{x_i \neq x_{lin}^*}  N_{x_i}(T) \big ] \label{eqn:regret_suboptimal_plays_upper_bound}
\end{eqnarray}

where $N_{x_i}(T)$ denotes the number of plays of arm with feature vector $x_i$ upto time $T$. Equation~\ref{eqn:regret_suboptimal_plays_upper_bound} follows from the fact that $\Delta_i \leq 1$ for all $i$ as the rewards are within $[0,1]$. Also,

\begin{eqnarray}
 \mathbb{E}_{lin} \big [ \sum_{x_i \neq x_{lin}^*} \Delta_i N_{x_i}(T) \big ] & \geq & \Delta \mathbb{E}_{lin} \big [ \sum_{x_i \neq x_{lin}^*}  N_{x_i}(T) \big ] \nonumber \\
 & = & \mathbb{E}_{lin} \big [ \sum_{x_i \neq x_{lin}^*}  N_{x_i}(T) \big ] \label{eqn:regret_suboptimal_plays_lower_bound}
\end{eqnarray}

From Equation~\ref{eqn:regret_suboptimal_plays_upper_bound} and \ref{eqn:regret_suboptimal_plays_lower_bound}, we have

\begin{equation}
\mathbb{E}_{lin}[R_{lin}(T)] = \mathbb{E}_{lin} \big [ \sum_{x_i \neq x_{lin}^*} N_{x_i}(T) \big ] \label{eqn:regret_suboptimal_plays}
\end{equation}

From the problem statement, algorithm \textbf{$\mathbb{A}$} suffers a regret of

\begin{equation}
\mathbb{E}_{lin} \big [ \sum_{x_i \neq x_{lin}^*} N_{x_i}(T) \big ] = O (d\sqrt{T}) \leq \alpha d\sqrt{T} \label{eqn:linear_bandit_upper_bound}
\end{equation}

where $\alpha \,\, (>0)$ is a constant. From (\ref{eqn:linear_bandit_upper_bound}), it is clear that there exists a sub-optimal arm with feature $x \neq x_{lin}^*$ such that,
\begin{equation}
\mathbb{E}_{lin} \big [ N_x(T) \big ] \leq \frac{\alpha d \sqrt{T}}{N-1}=\frac{\alpha d}{c} \nonumber
\end{equation}

From Markov's Inequality, we can write,
\begin{equation}
\mathbb{P}_{lin}(N_x(T) > a) \leq \frac{\alpha d}{c a} \nonumber
\end{equation}

Let $G_x$ be the event $\lbrace N_x (T) \leq a \rbrace$ and $G$ be the event $\lbrace R_{lin} (T) \leq  3 \alpha d \sqrt{T} \rbrace $. Choosing $a=3 d /c $, we have
\begin{equation}
\mathbb{P}_{lin}[G_x] \geq \frac{2}{3} \label{eqn:lower_bound_g_x}
\end{equation}

Also,
\begin{eqnarray}
\mathbb{P}_{lin}[G]&&=1-\mathbb{P}_{lin}[G^c] \nonumber \\
&&= 1- \mathbb{P}_{lin} \big [ R_{lin}(T) > 3 \alpha d \sqrt{T} \big ] \nonumber \\
&& \geq 1-\frac{d \alpha \sqrt{T}}{3 \alpha d \sqrt{T}}=2/3 
\label{eqn:lower_bound_g}
\end{eqnarray}

The last inequality follows directly from Markov's Inequality. From Equations~\ref{eqn:lower_bound_g_x} and \ref{eqn:lower_bound_g}, we have
\begin{equation}
\mathbb{P}_{lin}\big [ G \cap G_x \big ] \geq 1/3 \label{eqn:lower_bound_intersection}
\end{equation}

Now consider the perturbed model where arm with feature $x \neq x_{lin}^*$ is boosted to become optimal and the associated probability measure is denoted by $\mathbb{P}_{pert}[\cdot]$. We can write

\begin{eqnarray}
\mathbb{P}_{pert}[G] && \geq \mathbb{P}_{pert} \big [ G \cap G_x \big ] = \mathbb{E}_{pert} \big [ \mathrm{1}_{\lbrace G \cap G_x \rbrace} \big ] \nonumber \\
&&  \stackrel{\text{(1)}}{=} \mathbb{E}_{lin} \big [ \mathrm{1}_{\lbrace G \cap G_x \rbrace} \prod_{i=1}^{N_x(T)} \frac{\mathbb{P}_{pert}(r_{x,i})}{\mathbb{P}_{lin}(r_{x,i})} \big ] \nonumber \\
&& \stackrel{\text{(2)}}{\geq} \mathbb{E}_{lin} \big [\mathrm{1}_{\lbrace G \cap G_x \rbrace} h(a) \big ] \nonumber \\
&& = h(a) \mathbb{P}_{lin}\bigg [ G \cap G_x \big ] \nonumber\\
&& \geq (1/3) h(a) = \Omega(1)
\label{eqn:event_g_const_prob}
\end{eqnarray}

(1) follows from measure-change identity stated in Lemma~\ref{lem:measure_change_identity}. (2) can be shown as follows: $\prod_{i=1}^{N_x(T)} \frac{\mathbb{P}_{pert}(r_{x,i})}{\mathbb{P}_{lin}(r_{x,i})} \geq \prod_{i=1}^{N_x(T)} \mathbb{P}_{pert}(r_{x,i})$ and since under $G_x$, $N_x(T) \leq a$ and $\mathbb{P}_{pert}(r_{x,i}) \leq 1$, (2) follows with $h(a) =\prod_{i=1}^{a} \mathbb{P}_{pert}(r_{x,i})$. Note that the function $h(.)$ has no dependence on $T$.

In the perturbed model, the optimal arm is $x$, and hence the expected regret with respect to $x$ is,
\begin{eqnarray}
\mathbb{E}_{pert}[R_{pert}] && = \mathbb{E}_{pert}\big [ \sum_{x_i \neq x} N_{x_i}(T) \big ] \nonumber \\
&& \stackrel{\text{(1)}}{\geq} \mathbb{E}_{pert} \big [ \mathrm{1}_{\lbrace G \rbrace} \sum_{x_i \neq x} N_{x_i}(T) \big ] \nonumber \\
&& \stackrel{\text{(2)}}{\geq} \mathbb{E}_{pert} \big [ \mathrm{1}_{\lbrace G \rbrace} N_{x_{lin}^*}(T) \big ] \nonumber \\
&& \stackrel{\text{(3)}}{\geq} \mathbb{E}_{pert} \big [ \mathrm{1}_{\lbrace G \rbrace} (T-3\alpha d \sqrt{T}) \big ] \nonumber \\
&& \stackrel{\text{(4)}}{\geq} (T-3\alpha d \sqrt{T}) \Omega(1)\nonumber \\
&& = \Omega(T) \nonumber
\end{eqnarray}

where (1) is true as both $\mathrm{1}_{\lbrace G \rbrace}$ and $ \sum_{y \neq x} N_y(T)$ are non-negative random variables. (2) follows from the fact that under the perturbed model, $x_{lin}^*$ is an sub-optimal arm. Under the event $G$, $ N_{x_{lin}^*} \geq (T-3\alpha d \sqrt{T})$ and (3) follows directly. (4) is a direct consequence of Equation~\ref{eqn:event_g_const_prob}.

\section*{Proof of Theorem~\ref{lem:OFUL_large_dev}} 

Consider a linear bandit problem with
$\mathcal{A}=\lbrace 1,2 \rbrace$. Assume $d=1$, and the context matrix $\mathcal{X}=[1\; 2]$.
 Let the unknown parameter be $\theta^*$ and the mean vector be
$\mu=[\mu_1 \,\,\, \mu_2]^T$ with $\mu_2 > \mu_1$, (optimal arm is 2). All perfectly linear models lie on the solid line of Figure~\ref{fig:large_deviation}. The dotted line in Figure~\ref{fig:large_deviation} represents decision boundary. Consider the deviation vector to be $\epsilon= [\epsilon_1 \, \, \, \epsilon_2 ]^T$.

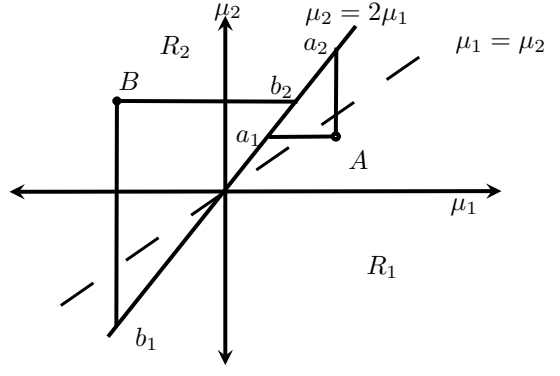
\begin{figure}[t!] 
\begin{centering}
\begin{center}

\ifx\du\undefined
  \newlength{\du}
\fi
\setlength{\du}{15\unitlength}
\begin{tikzpicture}
\pgftransformxscale{0.35000000}
\pgftransformyscale{-0.35000000}
\definecolor{dialinecolor}{rgb}{0.000000, 0.000000, 0.000000}
\pgfsetstrokecolor{dialinecolor}
\definecolor{dialinecolor}{rgb}{1.000000, 1.000000, 1.000000}
\pgfsetfillcolor{dialinecolor}
\pgfsetlinewidth{0.100000\du}
\pgfsetdash{}{0pt}
\pgfsetdash{}{0pt}
\pgfsetbuttcap
{
\definecolor{dialinecolor}{rgb}{0.000000, 0.000000, 0.000000}
\pgfsetfillcolor{dialinecolor}

\pgfsetarrowsstart{stealth}
\pgfsetarrowsend{stealth}
\definecolor{dialinecolor}{rgb}{0.000000, 0.000000, 0.000000}
\pgfsetstrokecolor{dialinecolor}
\draw (25.100000\du,2.200000\du)--(25.100000\du,27.400000\du);
}
\pgfsetlinewidth{0.100000\du}
\pgfsetdash{}{0pt}
\pgfsetdash{}{0pt}
\pgfsetbuttcap
{
\definecolor{dialinecolor}{rgb}{0.000000, 0.000000, 0.000000}
\pgfsetfillcolor{dialinecolor}

\pgfsetarrowsstart{stealth}
\pgfsetarrowsend{stealth}
\definecolor{dialinecolor}{rgb}{0.000000, 0.000000, 0.000000}
\pgfsetstrokecolor{dialinecolor}
\draw (9.550000\du,14.900000\du)--(45.000000\du,14.850000\du);
}
\pgfsetlinewidth{0.07500000\du}
\pgfsetdash{{1.000000\du}{1.000000\du}}{0\du}
\pgfsetdash{{1.000000\du}{1.000000\du}}{0\du}
\pgfsetbuttcap
{
\definecolor{dialinecolor}{rgb}{0.000000, 0.000000, 0.000000}
\pgfsetfillcolor{dialinecolor}

\definecolor{dialinecolor}{rgb}{0.000000, 0.000000, 0.000000}
\pgfsetstrokecolor{dialinecolor}
\draw (13.250000\du,23.100000\du)--(40.350000\du,4.300000\du);
}
\pgfsetlinewidth{0.100000\du}
\pgfsetdash{}{0pt}
\pgfsetdash{}{0pt}
\pgfsetbuttcap
{
\definecolor{dialinecolor}{rgb}{0.000000, 0.000000, 0.000000}
\pgfsetfillcolor{dialinecolor}

\definecolor{dialinecolor}{rgb}{0.000000, 0.000000, 0.000000}
\pgfsetstrokecolor{dialinecolor}
\draw (16.650000\du,25.350000\du)--(34.500000\du,2.950000\du);
}
\pgfsetlinewidth{0.100000\du}
\pgfsetdash{}{0pt}
\pgfsetdash{}{0pt}
\pgfsetbuttcap
\pgfsetmiterjoin
\pgfsetlinewidth{0.100000\du}
\pgfsetbuttcap
\pgfsetmiterjoin
\pgfsetdash{}{0pt}
\definecolor{dialinecolor}{rgb}{1.000000, 1.000000, 1.000000}
\pgfsetfillcolor{dialinecolor}
\pgfpathellipse{\pgfpoint{33.062500\du}{10.912500\du}}{\pgfpoint{0.212500\du}{0\du}}{\pgfpoint{0\du}{0.212500\du}}
\pgfusepath{fill}
\definecolor{dialinecolor}{rgb}{0.000000, 0.000000, 0.000000}
\pgfsetstrokecolor{dialinecolor}
\pgfpathellipse{\pgfpoint{33.062500\du}{10.912500\du}}{\pgfpoint{0.212500\du}{0\du}}{\pgfpoint{0\du}{0.212500\du}}
\pgfusepath{stroke}
\pgfsetbuttcap
\pgfsetmiterjoin
\pgfsetdash{}{0pt}
\definecolor{dialinecolor}{rgb}{0.000000, 0.000000, 0.000000}
\pgfsetstrokecolor{dialinecolor}
\pgfpathellipse{\pgfpoint{33.062500\du}{10.912500\du}}{\pgfpoint{0.212500\du}{0\du}}{\pgfpoint{0\du}{0.212500\du}}
\pgfusepath{stroke}
\pgfsetlinewidth{0.100000\du}
\pgfsetdash{}{0pt}
\pgfsetdash{}{0pt}
\pgfsetbuttcap
{
\definecolor{dialinecolor}{rgb}{0.000000, 0.000000, 0.000000}
\pgfsetfillcolor{dialinecolor}

\definecolor{dialinecolor}{rgb}{0.000000, 0.000000, 0.000000}
\pgfsetstrokecolor{dialinecolor}
\draw (33.062500\du,10.700000\du)--(33.100000\du,4.500000\du);
}
\pgfsetlinewidth{0.100000\du}
\pgfsetdash{}{0pt}
\pgfsetdash{}{0pt}
\pgfsetbuttcap
{
\definecolor{dialinecolor}{rgb}{0.000000, 0.000000, 0.000000}
\pgfsetfillcolor{dialinecolor}

\definecolor{dialinecolor}{rgb}{0.000000, 0.000000, 0.000000}
\pgfsetstrokecolor{dialinecolor}
\draw (32.850000\du,10.912500\du)--(28.250000\du,10.950000\du);
}
\pgfsetlinewidth{0.100000\du}
\pgfsetdash{}{0pt}
\pgfsetdash{}{0pt}
\pgfsetbuttcap
\pgfsetmiterjoin
\pgfsetlinewidth{0.100000\du}
\pgfsetbuttcap
\pgfsetmiterjoin
\pgfsetdash{}{0pt}
\definecolor{dialinecolor}{rgb}{1.000000, 1.000000, 1.000000}
\pgfsetfillcolor{dialinecolor}
\pgfpathellipse{\pgfpoint{17.290625\du}{8.353125\du}}{\pgfpoint{0.171875\du}{0\du}}{\pgfpoint{0\du}{0.171875\du}}
\pgfusepath{fill}
\definecolor{dialinecolor}{rgb}{0.000000, 0.000000, 0.000000}
\pgfsetstrokecolor{dialinecolor}
\pgfpathellipse{\pgfpoint{17.290625\du}{8.353125\du}}{\pgfpoint{0.171875\du}{0\du}}{\pgfpoint{0\du}{0.171875\du}}
\pgfusepath{stroke}
\pgfsetbuttcap
\pgfsetmiterjoin
\pgfsetdash{}{0pt}
\definecolor{dialinecolor}{rgb}{0.000000, 0.000000, 0.000000}
\pgfsetstrokecolor{dialinecolor}
\pgfpathellipse{\pgfpoint{17.290625\du}{8.353125\du}}{\pgfpoint{0.171875\du}{0\du}}{\pgfpoint{0\du}{0.171875\du}}
\pgfusepath{stroke}
\pgfsetlinewidth{0.100000\du}
\pgfsetdash{}{0pt}
\pgfsetdash{}{0pt}
\pgfsetbuttcap
{
\definecolor{dialinecolor}{rgb}{0.000000, 0.000000, 0.000000}
\pgfsetfillcolor{dialinecolor}

\definecolor{dialinecolor}{rgb}{0.000000, 0.000000, 0.000000}
\pgfsetstrokecolor{dialinecolor}
\draw (17.462500\du,8.353125\du)--(30.300000\du,8.400000\du);
}
\pgfsetlinewidth{0.100000\du}
\pgfsetdash{}{0pt}
\pgfsetdash{}{0pt}
\pgfsetbuttcap
{
\definecolor{dialinecolor}{rgb}{0.000000, 0.000000, 0.000000}
\pgfsetfillcolor{dialinecolor}

\definecolor{dialinecolor}{rgb}{0.000000, 0.000000, 0.000000}
\pgfsetstrokecolor{dialinecolor}
\draw (17.290625\du,8.525000\du)--(17.250000\du,24.600000\du);
}
\definecolor{dialinecolor}{rgb}{0.000000, 0.000000, 0.000000}
\pgfsetstrokecolor{dialinecolor}
\node[anchor=west] at (40.550000\du,16.000000\du){$\mu_1$};
\definecolor{dialinecolor}{rgb}{0.000000, 0.000000, 0.000000}
\pgfsetstrokecolor{dialinecolor}
\node[anchor=west] at (23.525000\du,1.892500\du){$\mu_2$};
\definecolor{dialinecolor}{rgb}{0.000000, 0.000000, 0.000000}
\pgfsetstrokecolor{dialinecolor}
\node[anchor=west] at (40.900000\du,4.300000\du){$\mu_1=\mu_2$};
\definecolor{dialinecolor}{rgb}{0.000000, 0.000000, 0.000000}
\pgfsetstrokecolor{dialinecolor}
\node[anchor=west] at (30.075000\du,2.085000\du){$\mu_2=2\mu_1$};
\definecolor{dialinecolor}{rgb}{0.000000, 0.000000, 0.000000}
\pgfsetstrokecolor{dialinecolor}
\node[anchor=west] at (33.200000\du,12.600000\du){$A$};
\definecolor{dialinecolor}{rgb}{0.000000, 0.000000, 0.000000}
\pgfsetstrokecolor{dialinecolor}
\node[anchor=west] at (34.500000\du,20.300000\du){$R_1$};
\definecolor{dialinecolor}{rgb}{0.000000, 0.000000, 0.000000}
\pgfsetstrokecolor{dialinecolor}
\node[anchor=west] at (19.575000\du,4.492500\du){$R_2$};
\definecolor{dialinecolor}{rgb}{0.000000, 0.000000, 0.000000}
\pgfsetstrokecolor{dialinecolor}
\node[anchor=west] at (16.625000\du,6.942500\du){$B$};
\definecolor{dialinecolor}{rgb}{0.000000, 0.000000, 0.000000}
\pgfsetstrokecolor{dialinecolor}
\node[anchor=west] at (25.075000\du,11.042500\du){$a_1$};
\definecolor{dialinecolor}{rgb}{0.000000, 0.000000, 0.000000}
\pgfsetstrokecolor{dialinecolor}
\node[anchor=west] at (29.9450000\du,4.437500\du){$a_2$};
\definecolor{dialinecolor}{rgb}{0.000000, 0.000000, 0.000000}
\pgfsetstrokecolor{dialinecolor}
\node[anchor=west] at (17.825000\du,25.532500\du){$b_1$};
\definecolor{dialinecolor}{rgb}{0.000000, 0.000000, 0.000000}
\pgfsetstrokecolor{dialinecolor}
\node[anchor=west] at (27.550000\du,7.327500\du){$b_2$};
\end{tikzpicture}

\end{center}
\end{centering}

\caption{A counter-example with $N=2$, $d=1$ with $\mathcal{X}=[1\; 2]$ and $\mu=[\mu_1 \,\,\, \mu_2]^T$. All linear models will lie on the solid line $\mu_2=2\mu_1$. The dotted line $\mu_2=\mu_1$ is the decision boundary which partitions the space spanned by $\mu_1$ and $\mu_2$ in 2 regions: $R_1$ (arm 1 is optimal) and $R_2$ (arm 2 is optimal). $A$ and $B$, not lying on $\mu_2=2\mu_1$ line, are taken for analyzing the behavior of OFUL.}
\label{fig:large_deviation}

\end{figure}

Now, consider the situation where the bandit model is non-linear with small deviation and mean vector $\mu$ is located at $A$. If  OFUL is played, the initial confidence interval will be $[a_1, \,\,\, a_2]$ (Figure~\ref{fig:large_deviation}) (obtained from the confidence set construction by OFUL). The confidence interval entirely lies in $R_2$  where the optimal action is to play arm $2$. Now, even if the player plays the sub-optimal arm (arm 1) for a few initial stages, the confidence interval will be concentrated around $a_1$ which is in region $R_2$ and thus the verdict of the OFUL will be to play arm $2$ which will result in zero regret at each step.

Now consider the scenario where the mean vector $\mu$ is located at $B$. The deviation vector $\epsilon$ is large with $|\epsilon_i| > c$ ($c >0$) for $i=1,2$. Thus for this instance, $l=c$ and $\beta=0$ and therefore $\mu$ is consistent with Definition~\ref{def:sparse}. The initial confidence interval obtained from the OFUL will be $[b_1,\,\,\, b_2]$, a part of which lies both in region $R_1$ and $R_2$. If the player starts playing arm $1$ (sub optimal) in the initial few rounds, the confidence interval starts concentrating around $b_1$ which lies in region $R_1$ and thus the suggestion of OFUL will be to play arm $1$. If player plays arm $1$, the confidence interval will further get concentrated around $b_1$ and the player will end up playing sub-optimal arm all the time. The regret incurred in this process is $\Omega(T)$ as the player incurs positive regret at each step.

The behavior of OFUL under the above setup of the counterexample can
be shown rigorously by the following formal proof:

In the current setup, we have $N=2$, $d=1$, context matrix $X= [x_1 \,\,\, x_2] = [1 \,\,\, 2]$, mean reward vector $[\mu_1 \,\,\, \mu_2]$ and unknown parameter $\theta^*$. Consider deviation vector $\epsilon= [\epsilon_1 \,  
\, \, \epsilon_2 ]^T$. Assuming linear model with deviation, $\mu_1 = \theta^* + \epsilon_1$ and $\mu_2= 2\theta^* + \epsilon_2$. We assume that $\mu_2 > \mu_1$, i.e., the optimal action is to play arm 2. 

Let us assume that, upto time $t$, arm $1$ and $2$ are played $N_1$ and $N_2$ times respectively (i.e., $N_1+N_2=t$). Based on the observations upto time $t$, if the learner plays OFUL, arm 2 will be played in the $t+1$ th instant if,

\begin{eqnarray}
&&\max_{\theta \in C_t} x_2^T\theta > \max_{\theta \in C_t} x_1^T \theta \nonumber \\
\Rightarrow &&\max_{\theta \in C_t} 2\theta > \max_{\theta \in C_t} \theta \label{eqn:OFUL_play_cond}
\end{eqnarray}

where the unknown parameter $\theta^*$ lies (with probability at least $1 - \tilde{\delta}$) in the confidence interval given by,
\footnotesize
 \begin{eqnarray}
&&C_t = \bigg \lbrace \theta \in \mathbb{R} : \norm {\hat{\theta}_t- \theta}_{\bar{V}_t} \leq  \lambda^{1/2}S +  \nonumber  \\
&& R \sqrt{2 \log  (\frac{\det(\bar{V}_t)^{1/2} \det (\lambda I)^{-1/2}}{\tilde{\delta}})} +  \norm {\mathbf{X}_{1:t}^T \epsilon_{1:t}}_{\bar{V}_t^{-1}} \bigg \rbrace \nonumber
\end{eqnarray}
\normalsize
where $\bar{V_t} = \lambda  + \sum_{s=1}^t x_s^2= \lambda + N_1 + 4N_2$, $\mathbf{X}_{1:t}^T=[x_{A_1},x_{A_2},\ldots,x_{A_t}]$, $\epsilon_{1:t}=[\epsilon_{A_1},\epsilon_{A_2},\ldots,\epsilon_{A_t}]^T$ and $A_i$ denote the action taken at time $i$. $\lambda$ denotes the regularization parameter and can be chosen in such a way that, $\lambda \ll N_1 +4N_2$. Thus, $\bar{V_t} \approx N_1 + 4N_2$.

Now, for $d=1$, the confidence interval $C_t$, 
\begin{equation}
\norm {\hat{\theta}_t- \theta}_{\bar{V}_t}= \sqrt{\bar{V_t}}|\theta - \hat{\theta_t}| \leq r \nonumber
\end{equation}
where, \\
\footnotesize
$r=\frac{1}{\sqrt{\bar{V_t}}} \bigg ( \lambda^{\frac{1}{2}}S+  R \sqrt{2 \log  (\frac{\bar{V}_t^{1/2}  \lambda^{-1/2}}{\tilde{\delta}})} +  \norm {\mathbf{X}_{1:t}^T \epsilon_{1:t}}_{\bar{V}_t^{-1}} \bigg ).$

 \begin{eqnarray}
 \norm {\mathbf{X}_{1:t}^T \epsilon_{1:t}}_{\bar{V}_t^{-1}} &= &\frac{1}{\bar{V_t}}|\sum_{i=1}^{t} x_{A_i}\epsilon_{A_i}| =  \frac{1}{\bar{V_t}} (N_1 \epsilon_1 + 4N_2 \epsilon_2) \nonumber
 \end{eqnarray}
 \normalsize

From Equation~\ref{eqn:OFUL_play_cond}, in order to play the sub-optimal arm (arm 1), the entire interval $[\hat{\theta_t}-r, \hat{\theta_t}+r]$ should lie in the negative side of the real line. Now, as $r>0$, it is sufficient to have $ \hat{\theta_t}+r < 0 $. Therefore,
\begin{eqnarray}
\mathbb{P}(\mbox{Play Arm 2})=1-\mathbb{P}(\mbox{Play Arm 1}) \nonumber \\
=1-\mathbb{P}(\hat{\theta_t}+r < 0 ) = \mathbb{P}(\hat{\theta_t}+r \geq 0) \nonumber
\end{eqnarray}

From the least square solution for $\hat{\theta_t}$,
\footnotesize
\begin{eqnarray}
\hat{\theta_t}&=&\frac{1}{(x_{A_1}^2+\ldots + x_{A_t}^2 + \lambda)} \bigg [(x_{A_1}^2+\ldots + x_{A_t}^2)\theta^* \nonumber \\  
&+& \sum_{i=1}^t x_{A_i}\epsilon_{A_i}+ \sum_{i=1}^t x_{A_i}\eta_i \bigg ] \nonumber \\
& \leq &\frac{N_1 + 4N_2}{N_1 +4N_2 + \lambda}\theta^* + \frac{N_1 \epsilon_1 + 4 N_2 \epsilon_2}{N_1 + 4N_2}+ \frac{\max x_{A_i}}{N_1 + 4N_2}\sum_{i=1}^t \eta_i \nonumber \\
&\approx &\theta^* +\frac{N_1 \epsilon_1 + 4 N_2 \epsilon_2}{N_1 + 4N_2}+ \frac{2}{N_1 + 4N_2}\sum_{i=1}^t \eta_i \nonumber
\end{eqnarray}
\normalsize

Following the argument of Section~\ref{subsec:linear_regret}, if arm 1 (sub-optimal arm) is played upto time $t$ almost all the time, we have $N_1 \rightarrow t$ and $N_2 \rightarrow 0$. Under such limiting condition,
\begin{equation}
\hat{\theta}_t \leq \theta^* + \epsilon_1 + \frac{2}{t}\sum_{i=1}^t \eta_i = \mu_1 + \frac{2}{t}\sum_{i=1}^t \eta_i \nonumber
\end{equation}
Therefore,
\begin{eqnarray}
&& \mathbb{P}(\hat{\theta_t} \geq -r) \leq \mathbb{P}(\mu_1 + \frac{2}{t}\sum_{i=1}^t \eta_i \geq -r) \nonumber \\
&& = \mathbb{P} \bigg (\frac{1}{t}\sum_{i=1}^t \eta_i > \frac{1}{2}(-r-\mu_1) \bigg ) \label{eqn:lower_bound_on_t}
\end{eqnarray}

Under the limiting condition, (i.e., $N_1 \rightarrow t$, $ N_2 \rightarrow 0$),
\begin{eqnarray}
r=\frac{1}{\sqrt{N_1}}\bigg (\lambda^{\frac{1}{2}} S + R\sqrt{2 \log (\frac{N_1^{\frac{1}{2}}}{\tilde{\delta} \lambda^{\frac{1}{2}}})} + \epsilon_1 \bigg ) \nonumber
\end{eqnarray}

As depicted in Figure~\ref{fig:large_deviation} (e.g. point B), for large deviation and small $\lambda$,
\begin{equation}
r \approx \frac{\epsilon_1}{\sqrt{N_1}} \approx \frac{\epsilon_1}{\sqrt{t}} \nonumber
\end{equation}

Now, consider point B of Figure~\ref{fig:large_deviation}. We observe that $\mu_1 < 0$. If $ t >  \frac{\epsilon_1^2}{\mu_1^2}\bydef T'$, the right hand side of the inequality in Equation~\ref{eqn:lower_bound_on_t} is positive and hence using Hoeffding's inequality for sub-gaussian random variables with zero mean, we have,
\begin{eqnarray}
\mathbb{P}(\hat{\theta_t} \geq -r) \leq \exp \bigg (-\frac {t (r+\mu_1)^2}{8R^2} \bigg ) \bydef \delta_{dev}(\epsilon_1)
\end{eqnarray}

As, $r$ is large, $\delta_{dev}(\epsilon_1)$ is considerably small and thus the probability of playing the optimal arm is significantly small. 

Hence, under the condition that the sub-optimal arm is played almost all the time for a minimum of $T'$ rounds (from the beginning), the learner is stuck with the sub-optimal arm with high probability.

When the sub-optimal arm (arm 1) is played, OFUL suffers a regret of $\mu_2 -\mu_1$, and thus the expected regret is,

\footnotesize
\begin{eqnarray}
\mathbb{E}(R_{OFUL}(T))&=& \mathbb{E}(R_{OFUL}(T'))+ \mathbb{E}(R_{OFUL}(T-T')) \nonumber \\
&=&(\mu_2 -\mu_1)T'+\mathbb{E}(R_{OFUL}(T-T')) \nonumber \\
 &\geq & (\mu_2 -\mu_1)T' + \sum_{l=T'+1}^{T} (1- \delta_{dev}(\epsilon_1))(\mu_2 -\mu_1) \nonumber
\end{eqnarray}
\normalsize
 which, when $\delta_{dev}(\epsilon_1)$ is very close to $0$ gives,
 \begin{equation}
 \mathbb{E}(R_{OFUL}(T))=\Omega(T). \nonumber
 \end{equation}
  
\section*{Auxiliary Lemmas for the Main Result (Theorem~\ref{thm:total_regret})} \label{subsec:perf_analysis}

We will analyze the performance of the RLB algorithm under $\mathcal{H}_0$ and $\mathcal{H}_1$. In particular we will show that the probability of deciding if favor of $\mathcal{H}_1$ where $\mathcal{H}_0$ is true, $\mathbb{P}(\mathcal{H}_1;\mathcal{H}_0)$, is very low. This is called the probability of false alarm. Similarly, if $\mathcal{H}_1$ is true, the intention is to show that $\mathbb{P}(\mathcal{H}_0;\mathcal{H}_1)$ is very small, or in other words, the miss detection is very small.

\subsection*{Analysis for $\mathcal{H}_0$:} \label{subsubsec:analysis_hyp0}

We use the notations of Section~\ref{subsec:conf_ellipsoid}. From Algorithm~\ref{algo:proposed_scheme}, RLB assumes a perfect linear model (i.e.  $\mathcal{H}_0$) to build the confidence interval. Lemma~(\ref{lem: conf_interval_hyp0}) provides the interval construction technique, center and length.

\begin{lemma}
\label{lem: conf_interval_hyp0}
Under $\mathcal{H}_0$, $\forall \theta \in C $, with $\hat{\theta}$ as an estimate of $\theta^*$, the projection of the Confidence Ellipsoid $C$ onto the context of $d+1$-th arm will result in an interval, $I_e= \big [ \left\langle x_{d+1}, \hat{\theta} \right\rangle - D \norm {x_{d+1}}_{\bar{V}^{-1}} \, , \, \left\langle x_{d+1}, \hat{\theta} \right\rangle + D \norm {x_{d+1}}_{\bar{V}^{-1}} \big ]$, where $D=R  \sqrt{2 \log  (\frac{\det(\bar{V})^{1/2} \det (\lambda I)^{-1/2}}{\tilde{\delta}})} + \lambda^{1/2}S $ and $\bar{V}=\lambda I + k \sum_{i=1}^{d} x_i x_i^T$.
\end{lemma}

\begin{proof}
In order to obtain the maximum value of the projection of $\theta$ onto the context vector corresponding to the $d+1$ th arm, $x_{d+1}$, where $\theta \in  C$,  we need to solve the following constrained optimization problem:
\begin{eqnarray}
\max_{\theta} \left\langle x_{d+1}, \theta \right\rangle  \nonumber \\ \nonumber
s.t. \norm {\theta - \hat{\theta}}_{\bar{V}} < D
\end{eqnarray}

The constraint can be re-written as, $\norm {\theta - \hat{\theta}}_{\bar{V}}^2 < D^2$. Define the Lagrangian $\mathcal{L}=\left\langle x_{d+1}, \theta \right\rangle + \gamma [\norm {\theta - \hat{\theta}}_{\bar{V}}^2 - D^2]$, where $\gamma$ is Lagrange multiplier. From the first order conditions, $\frac{\partial \mathcal{L}}{\partial \theta}=0$, we get $x_{d+1}=2\gamma \bar{V}(\theta - \hat{\theta})$, and $\gamma$ is computed by substituting the value of $x_{d+1}$ and pushing the inequality constraint to the exterior point, where it becomes an equality constraint.

Substituting the value of $\alpha$ and re-arranging the terms, we obtain the following solution:
\begin{equation}
\max_{\theta} \left\langle x_{d+1}, \theta \right\rangle =\left\langle x_{d+1}, \hat{\theta} \right\rangle +D\norm {x_{d+1}}_{\bar{V}^{-1}}  \label{eqn:max_projection}
\end{equation}

Similarly, in order to obtain the minimum length of projection, we have to solve, $\min_{\theta} \left\langle x_{d+1}, \hat{\theta} \right\rangle$ with same constraint and solution will be given by:

\begin{equation}
\min_{\theta} \left\langle x_{d+1}, \theta \right\rangle =\left\langle x_{d+1}, \theta \right\rangle -D\norm {x_{d+1}}_{\bar{V}^{-1} } \label{eqn:min_projection}
\end{equation}

Combining Equations~(\ref{eqn:max_projection}) and (\ref{eqn:min_projection}), the result follows.
\end{proof}

 Now, with respect to the algorithmic description in Section~\ref{subsec:conf_ellipsoid}, we will compare the interval resulting from Hoeffding's inequality (based on the mean estimate of the reward of $d+1$ th arm), $I_s$ with this projected interval $ I_e$. $I_e$ and $I_s $ are both symmetric intervals centered at $\left\langle x_{d+1}, \hat{\theta} \right\rangle$ and $\hat{\mu}_{d+1}=\left\langle x_{d+1}, \theta^* \right\rangle $ and having radius $D\norm {x_{d+1}}_{\bar{V}^{-1}} \bydef r_p $ and $r_s$ ($= \sqrt{\frac{\log(1/\delta_s)}{2k}}$, Section~\ref{subsec:sampling_interval}) respectively. Motivated by the UCB algorithm, the intervals $I_e$ and $I_s$ are boosted by a factor of $\sqrt{\log k}$. This is done in order to have a control over the probability of false alarm and miss alarm via parameter $k$. The resulting new radii of the intervals are $\tilde{r}_p \bydef r_p \sqrt{\log k}$ and $\tilde{r}_s \bydef r_s \sqrt{\log k}$ respectively. Lemma~(\ref{lem:low_false_alarm}) provides a theoretical guarantee for a very low and tune-able (as a function of $k$)  $\mathbb{P}(\mathcal{H}_1;\mathcal{H}_0)$.

\begin{lemma}
Under $\mathcal{H}_0$, for any $\delta_1(k,\lambda) >0$,

\begin{eqnarray}
&&\mathbb{P}(I_e \cap I_s = \phi) \leq \delta_1(k,\lambda), \nonumber \\
&& \delta_1(k,\lambda) \bydef \exp \bigg( - \frac{k (\tilde{r}_s + r_p (\sqrt{\log k} -1))^2}{2R^2} \bigg ). \nonumber 
\end{eqnarray}

\label{lem:low_false_alarm}
\end{lemma}

\begin{proof}

 If the intervals $I_e$ and $I_s$ are non-overlapping, the distance between the centers should be more than the sum of radius of each intervals. Thus,
\begin{equation}
\mathbb{P}(I_e \cap I_s = \phi) = \mathbb{P}\bigg (|\left\langle x_{d+1}, \hat{\theta} \right\rangle - \hat{\mu}_{d+1}| > \tilde{r}_p + \tilde{r}_s \bigg ) \nonumber
\end{equation}

Under $\mathcal{H}_0$, $ \hat{\theta}$ and $\hat{\mu}_{d+1}$ are computed as follows:
\begin{eqnarray}
\hat{\theta} &=& (\mathbf{X}^T \mathbf{X} + \lambda I)^{-1}\mathbf{X}^T \mathbf{Y} \nonumber \\
&=& (\mathbf{X}^T \mathbf{X} + \lambda I)^{-1}\mathbf{X}^T (\mathbf{X} \theta^* + \eta) \nonumber \\
&=& (\mathbf{X}^T \mathbf{X} + \lambda I)^{-1}\mathbf{X}^T \eta -\lambda (\mathbf{X}^T \mathbf{X} + \lambda I)^{-1} \theta^* + \nonumber \\
&& + (\mathbf{X}^T \mathbf{X} + \lambda I)^{-1} (\mathbf{X}^T \mathbf{X} + \lambda I) \theta^* \nonumber \\
&=& (\mathbf{X}^T \mathbf{X} + \lambda I)^{-1}\mathbf{X}^T \eta + \theta^* -\lambda (\mathbf{X}^T \mathbf{X} + \lambda I)^{-1} \theta^* \nonumber
\end{eqnarray}

\begin{eqnarray}
\hat{\mu}_{d+1} &=& \frac{1}{k}\sum_{i=1}^{k} \bigg ( \left\langle x_{d+1}, \theta^* \right\rangle + \eta_i \bigg ) \nonumber \\
&=& \left\langle x_{d+1}, \theta^* \right\rangle + \frac{1}{k} \sum_{i=1}^{k} \eta_i \nonumber
\end{eqnarray}

Therefore,

\begin{eqnarray}
\footnotesize
&& \left\langle x_{d+1}, \hat{\theta} \right\rangle - \hat{\mu}_{d+1} = x_{d+1}^T \bigg [ (\mathbf{X}^T \mathbf{X} + \lambda I)^{-1} \mathbf{X}^T \eta  \nonumber \\
&& + \theta^* -\lambda  (\mathbf{X}^T \mathbf{X} + \lambda I)^{-1} \theta^* - \theta^* \bigg ] - \frac{1}{k}\sum_{i=1}^{k} \eta_i \nonumber \\
&=& \left\langle x_{d+1}, \mathbf{X}^T \eta \right\rangle_{\bar{V}^{-1}} -\lambda \left\langle x_{d+1},  \theta^* \right\rangle_{\bar{V}^{-1}} - \frac{1}{k}\sum_{i=1}^{k} \eta_i \nonumber \\
&\leq & \norm {x_{d+1}}_{\bar{V}^{-1}} (\norm {\mathbf{X}^T \eta}_{\bar{V}^{-1}} + \lambda \norm {\theta^*}_{\bar{V}^{-1}})-\frac{1}{k}\sum_{i=1}^{k} \eta_i \nonumber \\
& \leq & \norm {x_{d+1}}_{\bar{V}^{-1}} \bigg ( \norm {\mathbf{X}^T \eta}_{\bar{V}^{-1}} + \lambda^{1/2}\norm {\theta^*}_2 \bigg ) -\frac{1}{k}\sum_{i=1}^{k} \eta_i \nonumber \\
& \leq & \norm {x_{d+1}}_{\bar{V}^{-1}} \bigg ( \norm {\mathbf{X}^T \eta }_{\bar{V}^{-1}} + \lambda^{1/2}S \bigg ) -\frac{1}{k}\sum_{i=1}^{k} \eta_i \nonumber \\
& \leq & r_p -\frac{1}{k}\sum_{i=1}^{k} \eta_i \nonumber
\end{eqnarray}

The first inequality is a consequence of Cauchy Schwartz Inequality. For the second inequality, the following property is used \cite{AbbYadPS11}:

\begin{equation}
\norm {\theta^*}_{\bar{V}^{-1}}^2 \leq \frac{1}{\lambda_{\min}(\bar{V})}\norm {\theta^*}_2^2 \leq \frac{1}{\lambda}\norm {\theta^*}_2^2 \leq \frac{1}{\lambda} S^2 
\end{equation}

The last inequality is a consequence of Theorem (1) of \cite{AbbYadPS11}, for any $\tilde{\delta} > 0$, with probability at least $1-\tilde{\delta}$,
\begin{equation}
\norm {\mathbf{X}^T \eta}_{\bar{V}^{-1}} \leq R  \sqrt{2 \log  (\frac{\det(\bar{V})^{1/2} \det (\lambda I)^{-1/2}}{\tilde{\delta}})} \nonumber
\end{equation}

Thus,
\begin{eqnarray}
\footnotesize
&&\mathbb{P}\bigg (\left\langle x_{d+1}, \hat{\theta} \right\rangle - \hat{\mu}_{d+1} > \tilde{r}_p + \tilde{r}_s \bigg ) \nonumber \\
&\leq & \mathbb{P}\bigg (r_p -\frac{1}{k}\sum_{i=1}^{k} \eta_i \geq \tilde{r}_p + \tilde{r}_s \bigg ) \nonumber \\
&=& \mathbb{P}\bigg ( -\frac{1}{k}\sum_{i=1}^{k} \eta_i \geq \tilde{r}_s + r_p (\sqrt{\log k} -1) \bigg ) \nonumber
\end{eqnarray}

 $\eta_i, i=1,2,\ldots,k $ are independent and sampled from conditionally R-sub gaussian distribution with mean $0$. Using the Hoeffding's inequality for sub-gaussian random variables (assuming $k > e$, base of natural logarithm),

\begin{eqnarray}
\mathbb{P}\bigg ( -\frac{1}{k}\sum_{i=1}^{k} \eta_i \geq \tilde{r}_s + r_p (\sqrt{\log k} -1) \bigg ) \nonumber \\
 \leq \exp \bigg( - \frac{k (\tilde{r}_s + r_p (\sqrt{\log k} -1))^2}{2R^2} \bigg ) \nonumber
\end{eqnarray}

Define $\delta_1(k,\lambda) \bydef \exp \bigg( - \frac{k (\tilde{r}_s + r_p (\sqrt{\log k} -1))^2}{2R^2} \bigg )$ and the result follows.

Using the same line of argument, it can be shown that $\mathbb{P}\bigg ( \hat{\mu}_{d+1}-\left\langle x_{d+1}, \hat{\theta} \right\rangle  <-( \tilde{r}_p + \tilde{r}_s) \bigg ) \leq \delta_1(k,\lambda)$ and thus combining together,

\begin{equation}
\mathbb{P}\bigg (| \left\langle x_{d+1}, \hat{\theta} \right\rangle- \hat{\mu}_{d+1}| > \tilde{r}_p + \tilde{r}_s \bigg ) \leq \delta_1(k,\lambda).   \nonumber
\end{equation}
\end{proof}

\subsection*{Analysis for $\mathcal{H}_1$:} \label{subsubsec:analysis_hyp1}

Similarly,  we study the performance of Robust Linear Bandit under $\mathcal{H}_1$ and Lemma~(\ref{lem:low_miss_prob}) provides a theoretical guarantee on the probability of miss alarm $\mathbb{P}(\mathcal{H}_0;\mathcal{H}_1)$. Context vector $\mathbf{X}$, reward vector $\mathbf{Y}$ and noise vector $\eta$ will be defined similar to Section~\ref{subsec:conf_ellipsoid}. Additionally, we define deviation vector $\epsilon '=[ \epsilon_1, \,\, \epsilon_1, \ldots ,\epsilon_1, \,\, ,\epsilon_2 \ldots, \epsilon_2, \ldots, \epsilon_d]^T$ where each $\epsilon_i$, $i=1,2,\ldots,d$ is repeated $k$ times. The reward vector is denoted by, $\mathbf{Y}= \mathbf{X}\theta + \epsilon'+ \eta$. RLB always assumes a linear model while constructing the confidence interval and thus neglects $\epsilon'$. If data comes from $\mathcal{H}_1$, the discrepancy in the algorithm can be caught by comparing the confidence interval, $I_e$, built by RLB and the interval $I_s$ obtained from data samples. The information about the deviation $\epsilon'$ will be embedded in $I_s$ and thus, by Lemma~\ref{lem:low_miss_prob}, with high probability it will not intersect with $I_e$. 

\begin{lemma}

Under $\mathcal{H}_1$, for any $\delta_2 (k,\lambda) > 0$

\begin{eqnarray}
&&\mathbb{P}(I_e \cap I_s \neq \phi) \leq \delta_2(k,\lambda)+\beta,  \nonumber \\
&&\delta_2(k,\lambda) \bydef \exp (- \frac{k (l_1 -\tilde{r}_p -\tilde{r}_s)^2 }{2R^2}). \nonumber 
\end{eqnarray}
where $ l_1 > \sqrt{\log k}(r_p + r_s) $, $\beta > 0$.
\label{lem:low_miss_prob}

\end{lemma}

\begin{proof}
Define the following event:
\newline
 $\mathcal{A}= \big \lbrace  |x_{i_{d+1}}^T[\chi_{i_1,\ldots,i_d}]^{-1}[\mu_{i_1,\ldots,i_d}]-\mu_{i_{d+1}}| \geq l \big \rbrace$.
 
 for all choices of $d+1$ arms out of $N$ arms. Initially, we will assume that $\mathcal{A}$ is true.

 If the intervals $I_e$ and $I_s$ intersect, the distance between the centers should be strictly less than the sum of radius of each intervals. The radius of the intervals will be $\tilde{r}_p$ and $\tilde{r}_s$. Thus,

\begin{equation}
\mathbb{P}(I_e \cap I_s \neq \phi) = \mathbb{P}\bigg (|\left\langle x_{d+1}, \hat{\theta} \right\rangle - \hat{\mu}_{d+1}| < \tilde{r}_p + \tilde{r}_s \bigg ) \nonumber
\end{equation}

The algorithm computes $\hat{\theta}$ using the perfect linear model, and using the samples from the first $d$ arms. Under $\mathcal{H}_0$, $\left\langle x_{d+1}, \hat{\theta} \right\rangle$ is an estimate of $\mu_{d+1}$. It is assumed that the deviation vector $\epsilon$ is non-sparse (Definition~\ref{def:sparse}). Thus, under $\mathcal{H}_1$,

\begin{equation}
|\left\langle x_{d+1}, \hat{\theta} \right\rangle - \mu_{d+1}| > l_1 >0  \label{eqn:deviation_hyp1}
\end{equation}

Therefore,
\begin{eqnarray}
&&\mathbb{P}\bigg (\left\langle x_{d+1}, \hat{\theta} \right\rangle - \hat{\mu}_{d+1} < \tilde{r}_p + \tilde{r}_s \bigg ) \nonumber \\
&&=\mathbb{P}\bigg (\left\langle x_{d+1}, \hat{\theta} \right\rangle -\mu_{d+1} +\mu_{d+1}- \hat{\mu}_{d+1} < \tilde{r}_p + \tilde{r}_s \bigg ) \nonumber 
\end{eqnarray}

From Equation~(\ref{eqn:deviation_hyp1}),
\begin{equation}
\left\langle x_{d+1}, \hat{\theta} \right\rangle -\mu_{d+1} +\mu_{d+1}- \hat{\mu}_{d+1} > l_1 + (\mu_{d+1}- \hat{\mu}_{d+1}) \nonumber
\end{equation}

Substituting,
\begin{eqnarray}
\mathbb{P}\bigg (\left\langle x_{d+1}, \hat{\theta} \right\rangle - \hat{\mu}_{d+1} < \tilde{r}_p + \tilde{r}_s \bigg ) \nonumber \\
< \mathbb{P}\bigg (  \hat{\mu}_{d+1}-\mu_{d+1} > l_1 -\tilde{r}_p - \tilde{r}_s \bigg ) \nonumber \\
\leq \exp \bigg (- \frac{k (l_1 -\tilde{r}_p - \tilde{r}_s)^2 }{2R^2}   \bigg ) \nonumber
\end{eqnarray}
 
where the last inequality is true if $ l_1 > \tilde{r}_p + \tilde{r}_s $. Define $\delta_2(k,\lambda)=\exp \bigg (- \frac{k (l_1 -\tilde{r}_p - \tilde{r}_s)^2 }{2R^2} \bigg ) $. Using the same line of arguments, we can show,
\begin{equation}
\mathbb{P}\bigg (\hat{\mu}_{d+1} - \left\langle x_{d+1}, \hat{\theta} \right\rangle  > -(\tilde{r}_p + \tilde{r}_s)\bigg ) \leq \delta_2 (k,\lambda) \nonumber
\end{equation}

From Definition~\ref{def:sparse}, $\mathbb{P}(\mathcal{A})\geq 1-\beta$. Therefore using union bound, we have

\begin{equation}
\mathbb{P}\bigg (|\hat{\mu}_{d+1} - \left\langle x_{d+1}, \hat{\theta} \right\rangle|  < (\tilde{r}_p + \tilde{r}_s)\bigg ) \leq \delta_2 (k,\lambda)+\beta \nonumber
\end{equation}
\end{proof}

  \section*{Proof of Theorem~\ref{thm:total_regret}} 

Under $\mathcal{H}_0$, the algorithm chooses perturbed OFUL and UCB with probability $1-\delta_1(k,\lambda)$ and $\delta_1(k,\lambda)$ respectively. The upper regret bound for OFUL and UCB is given in Section~\ref{sec:regret}. Thus the expected regret of the proposed algorithm $\mathbb{E}(R_{RLB})(T)$ will follow Theorem~\ref{thm:total_regret}.

Also if $\mathcal{H}_1$ is true, using the same line of argument, $\mathbb{E}(R_{RLB})$ will follow Theorem~\ref{thm:total_regret}.


\end{document}